\documentclass[11pt]{article}

\usepackage[margin=1in]{geometry}
\usepackage{amsmath, amssymb, amsthm}
\usepackage{graphicx}
\usepackage{caption}
\usepackage{subcaption}
\usepackage{booktabs}
\usepackage{float}
\usepackage{newtxtext,newtxmath}
\usepackage{fancyhdr}
\usepackage{tikz}
\usepackage{titlesec}
\usepackage{multicol}
\usepackage{dblfloatfix}
\usepackage{placeins}
\usepackage[authoryear,round]{natbib}
\usepackage[final]{microtype}
\usepackage{mathtools}   
\usepackage{bm}          

\AtBeginDocument{\fontsize{11}{13}\selectfont}

\usepackage[hyphens]{url}  
\usepackage{xurl}          
\makeatletter
\g@addto@macro{\UrlBreaks}{\do\-}   
\makeatother

\usepackage{etoolbox}
\usepackage{hyperref}
\hypersetup{breaklinks=true}

\titleformat{\section}{\large\bfseries}{\thesection.}{0.5em}{}
\titleformat{\subsection}{\normalsize\bfseries}{\thesubsection.}{0.5em}{}

\theoremstyle{plain}  

\newtheorem{theorem}{Theorem}[section]  
\newtheorem{proposition}{Proposition}[section]

\newtheorem{corollary}{Corollary}[section]

\usepackage{subcaption}
\newcommand{\topicstart}[1]{\smallskip\noindent\textbf{#1}}
\usepackage{orcidlink}
\usepackage{fontawesome5}  

\newenvironment{acknowledgment}[1][Acknowledgment]{%
  \section*{#1}
  \addcontentsline{toc}{section}{#1}
}{}

\newenvironment{keywords}[1][Keywords]{%
  \medskip\small\noindent\textbf{#1: }\ignorespaces
}{%
  \par\normalsize
}

\title{Closed-Form Beta Distribution Estimation from Sparse Statistics with Random Forest Implicit Regularization}

\author{\Large Jonathan R. Landers\\[1em]
\faEnvelope\ \href{mailto:jonathan.robert.landers@gmail.com}{\texttt{jonathan.robert.landers@gmail.com}}\\[2pt]
\orcidlink{0000-0003-1872-6179}\ \href{https://orcid.org/0000-0003-1872-6179}{\texttt{orcid.org/0000-0003-1872-6179}}\\[1pt]
\faGithub\ \href{https://github.com/jonland82/seatgeek-beta-modeling}{\texttt{github.com/jonland82/seatgeek-beta-modeling}}
}

\date{}

\begin{document}
\maketitle

\begin{abstract}
This work advances distribution recovery from sparse data and ensemble classification through three main contributions. First, we introduce a closed-form estimator that reconstructs scaled beta distributions from limited statistics (minimum, maximum, mean, and median) via composite quantile and moment matching. The recovered parameters $(\alpha,\beta)$, when used as features in Random Forest classifiers, improve pairwise classification on time-series snapshots, validating the fidelity of the recovered distributions. Second, we establish a link between classification accuracy and distributional closeness by deriving error bounds that constrain total variation distance and Jensen-Shannon divergence, the latter exhibiting quadratic convergence. Third, we show that zero-variance features act as an implicit regularizer, increasing selection probability for mid-ranked predictors and producing deeper, more varied trees. A SeatGeek pricing dataset serves as the primary application, illustrating distributional recovery and event-level classification while situating these methods within the structure and dynamics of the secondary ticket marketplace. The UCI handwritten digits dataset confirms the broader regularization effect. Overall, the study outlines a practical route from sparse distributional snapshots to closed-form estimation and improved ensemble accuracy, with reliability enhanced through implicit regularization.
\end{abstract}

\begin{center}
\end{center}

\section{Introduction}
\label{sec:introduction}

Recovering probability distributions from limited information is a central problem in data analysis. In many applied settings, only a small set of summaries (minimum, maximum, mean, and median) is available rather than full samples. For time-series classification, target examples can be compared through compact snapshots of representative, underlying distributions. In this context, ensemble methods such as Random Forests benefit from features derived from reconstructed distributions, provided estimation is tractable and theoretically grounded. In practice, we interpret out-of-sample Random Forest accuracy as a pragmatic gauge of distributional integrity.

\topicstart{This paper presents the following contributions:}

\begin{enumerate}
  \item \textbf{Closed-form distribution estimation from limited statistics.} A method is introduced to reconstruct scaled beta laws from incomplete statistics using composite quantile and moment matching, producing parameters $(\alpha,\beta)$ that convey shape information beyond location summaries (Sections~\ref{sec:scaled_beta_distribution} and~\ref{sec:quantile_moment_matching}). Injecting $(\alpha,\beta)$ into Random Forests improves pairwise classification of time-series snapshots, indicating that the recovered distributions preserve class-distinctive structure (Sect.~\ref{sec:random_forest_results}). Recovery is not only theoretically sound but also efficient, with a closed-form estimator that contrasts with iterative alternatives.
  \item \textbf{Accuracy–fidelity theory.} A link between predictive accuracy and distributional fidelity is established: Theorem~\ref{thm:total_variation} (\emph{Classification Accuracy and Total Variation Distance}) bounds total variation distance by classification error, and Theorem~\ref{thm:jensen_shannon} (\emph{Classification Accuracy and Jensen-Shannon Divergence}) strengthens the connection to a quadratic convergence rate in Jensen-Shannon divergence, showing that modest accuracy gains imply disproportionately larger reductions in information-theoretic divergence. Accordingly, classifier performance can be used as an operational proxy for how closely the recovered distributions match the (unobserved) ground truth.
  \item \textbf{Implicit regularization via zero-variance features.} When zero-variance (constant-value) features are added to the Random Forest ensemble, Theorem~\ref{thm:zero_variance_dilution} (\emph{Zero-Variance Dilution Effect}) formalizes how split-selection probabilities are rebalanced away from over-dominant predictors. Corollary~\ref{corollary:increased_tree_depth} (\emph{Increased Expected Tree Depth}) and Corollary~\ref{corollary:reduced_ensemble_correlation} (\emph{Reduced Ensemble Correlation}) show that this dilution yields deeper, more varied trees and lowers inter-tree correlation. Theorem~\ref{thm:continuous_approx} (\emph{Continuous Approximation via Zero-Variance Dilution}) together with Corollary~\ref{corollary:continuous_accuracy} (\emph{Continuous Accuracy Expansion via Selection Probability}) then show that the mechanism enables fine-grained control over selection probabilities and tree variety. Although accuracy gains are intentionally modest (Sect.~\ref{sec:regularization_experimental_results}), the quadratic bound ensures that they correspond to meaningful improvements in distributional correspondence, reinforcing trust in recovery fidelity.
\end{enumerate}

At a high level, the narrative arc is: sparse distributional snapshot from time series $\rightarrow$ closed-form scaled beta estimation $(\alpha,\beta)$ $\rightarrow$ fidelity evidenced by Random Forest classification gains and amplified by implicit regularization from constant-value features.

The empirical study centers on two datasets. The primary application is a curated SeatGeek time-series dataset collected via the SeatGeek API from May 2023 to May 2024 \citep{seatgeek2025}. It contains approximately 130,000 events, 15,400 artists or acts, and 6,700 venues across the United States, spanning globally recognized performers (e.g., Metallica, Taylor Swift) as well as local acts. The dataset reflects a large and growing U.S. online ticket market with revenue for online ticket sales estimated at \$12.7 billion in 2024 \citep{CRS2025}. SeatGeek’s position in this market has been reinforced through strategic partnerships. Paciolan, a leading ticketing provider, selected the company as the official secondary marketplace for its college athletics clients in February 2023, with integration beginning in July 2023 \citep{SBJ2023}. The ability to explicitly reverse-engineer event pricing distributions not only provides practical value for market applications, but also holds methodological interest for a wide range of domains where inference from limited statistics is required. Additional experiments on the UCI ML handwritten digits dataset \citep{uci_digits} show that the implicit regularization mechanism is not specific to ticketing data and carries over to a standard benchmark.

The remainder of the paper proceeds as follows. 
Section~\ref{sec:related_work} reviews prior work on distribution estimation from limited statistics, quantile- and moment-based inference, time-series classification, and implicit regularization in ensembles. 
Section~\ref{sec:preliminaries} describes the SeatGeek dataset, the raw time-series representation, and the transformation to distributional features. 
Section~\ref{sec:estimation} details the scaled beta estimation based on composite quantile and moment matching and states the formal accuracy–fidelity bounds. 
Section~\ref{sec:regularization} develops the implicit regularization formalism and provides supporting evidence. 
Section~\ref{sec:conclusions} concludes with implications and outlook.

\section{Related work}
\label{sec:related_work}

The literature relevant to this research spans areas that directly correspond to our contributions: (1) distribution-based parameter estimation using limited statistics, (2) statistical learning theory connecting estimation accuracy with classification performance, and (3) implicit regularization in ensemble methods. We briefly review each in turn.

\subsection{Distribution-based parameter estimation and classification}
Estimating parameters without full distributions and classifying time series are long-standing problems across machine learning, statistics, and econometrics. We estimate scaled beta parameters via composite quantile and moment matching from limited summaries, connecting to quantile-based estimation, moment matching, statistical learning theory, and feature-based TSC. Classic TSC baselines such as DTW/1-NN~\citep{berndt1994using,rakthanmanon2012searching} and a recent survey~\citep{middlehurst2024bakeoff} frame our comparisons, including ensemble methods like TSF and CIF. Shapelets~\citep{ye2009time} introduce discriminative subsequences; in our setting, subsequences aid distribution recovery that ultimately supports act classification. For beta distributions, \citep{krishnamoorthy2016handbook} covers traditional fitting under full samples; we extend to estimators from limited summaries, enabling event-specific modeling with minimal data.

Feature toolkits such as TSFresh~\citep{christ2018tsfresh} and Catch22~\citep{lubba2019catch22} extract broad or minimal sets of interpretable statistics; our approach instead learns a bounded-support distribution with few parameters. Quantile-centric and bounded-support works further motivate this stance: Quantile Flows for GFlowNets~\citep{zhang2023quantile} show how quantiles can replace point estimates; Beta Diffusion~\citep{zhou2023beta} highlights beta’s flexibility for range-bounded inference; QUANT~\citep{dempster2024quant} attains SOTA TSC using only quantiles from dyadic intervals; LQM~\citep{wei2024latent} demonstrates that limited quantiles preserve properties needed downstream; and a black-box simulation framework~\citep{lenzi2025blackbox} targets parameter recovery under limited information.

Beyond quantiles, moment-matching methods inform our estimators. Moment matching accelerates diffusion sampling by aligning conditional expectations~\citep{salimans2024moment} and improves denoising Gibbs sampling in energy-based models~\citep{zhang2023gibbs}. Reliability estimation for the exponentiated Pareto distribution from only lower record values~\cite{pareto2025reliability} similarly infers parameters from severely limited statistics. In this spirit, we combine quantile- and moment-based constraints to recover scaled beta parameters for ticket pricing when observations are sparse.

\subsection{Learning theory and estimation accuracy}
\citep{lin1991divergence} introduced the Jensen-Shannon divergence as a symmetrical, bounded  
measure of distributional distance, demonstrating how classification accuracy can deteriorate  
significantly when estimated and true distributions diverge. \citep{devroye1996probabilistic} then provided probabilistic bounds on classification risk,  
directly linking distribution-estimation error to predictive accuracy. 
\citep{tsybakov2004optimal} introduced margin conditions under which classification error  
rates converge optimally, establishing a deeper connection between parameter-estimation precision and  
classification performance.

Our findings follow these foundational insights: improved estimation of beta parameters leads to
more accurate classification of event types, while misestimation propagates into downstream
classification error. By mapping the proposed parameter-estimation method to these theoretical
frameworks, we demonstrate how precisely characterizing the underlying distribution supports
robust predictive performance. Moreover, the observed classification accuracy itself provides
indirect validation that the estimated distributions faithfully capture key aspects of the true
underlying pricing dynamics. In this way, our scaled beta approach echoes the broader principle
in statistical learning that well-characterized data distributions are essential for achieving
strong generalization, and conversely, strong generalization serves as empirical evidence of
distributional fidelity.

\subsection{Implicit regularization and entropy in random forests}
Work on ensembles, especially Random Forests, has long emphasized implicit regularization for robust generalization. Early foundations include Jensen-Shannon divergence as a tool for measuring distributional shifts~\citep{lin1991divergence}, bagging for variance reduction~\citep{breiman1996bagging}, the Random Subspace Method (RSM) to limit over-reliance on any feature subset~\citep{ho1998rsm}, and Extremely Randomized Trees, which further inject randomness into splits~\citep{geurts2006extremely}.

\citep{breiman2001random} formalized generalization in terms of tree “strength” and inter-tree “correlation,” showing that lower correlation improves accuracy. Stability and entropy-based perspectives~\citep{bousquet2002stability} complement this view. Despite being capable of interpolation, Random Forests can generalize via ensemble self-averaging~\citep{wyner2017explain}. In causal forests, adaptive neighborhood selection provides implicit regularization that reduces estimation variance~\citep{wager2018causal}. Random feature selection ($\mathtt{max\_features}$ or $m$) likewise lowers variance and acts as an implicit regularizer~\citep{mentch2019quantifying}, a theme connected to budget-aware hyperparameter tuning~\citep{cironis2022automatic}. Relatedly, sparse Bayesian learning with automatic-weighting Laplace priors shows how structural constraints induce implicit regularization~\citep{bai2023sbl}.

Our contribution highlights an additional, often overlooked mechanism: introducing zero-variance features reshapes the feature-selection distribution, acting as an “entropy-based stabilizer.” By reducing the probability that dominant predictors monopolize splits, this increases ensemble variety, can promote deeper or more varied trees, and reduces inter-tree correlation. In our scaled beta setting with limited summaries, such redistribution ensures newly introduced distributional parameters ($\alpha,\beta$) are not overshadowed by obvious predictors and can inform split decisions.

Regularization by explicit penalties is classical, e.g., ridge regression and Tikhonov regularization~\citep{hoerl1970ridge,tikhonov1943stability}. From a KL/entropy viewpoint, Random Forests, while not Bayesian, can still exhibit entropy-driven smoothing akin to the stability perspective of~\citep{bousquet2002stability}. The probability-redistribution effect parallels function-smoothing in FDA; for instance, roughness penalties in free-knot spline estimation~\citep{magistris2024} avoid over-concentration and preserve balanced structure. Finally, the link we establish among implicit regularization, feature-selection probabilities, and classification accuracy resonates with recent work formulating hyperparameter optimization for randomized algorithms as a stochastic inverse problem solved via Ensemble Kalman Inversion~\citep{dunbar2025hyperparameter}.

\section{Preliminaries and data}
\label{sec:preliminaries}

\subsection{Event time-series data}
The main use-case dataset consists of daily snapshots of secondary concert ticket prices collected through the SeatGeek API, covering the period from May 2023 to May 2024. In total it includes approximately 130,000 events, 15,400 artists or acts, and 6,700 venues across the United States. For each event, price information was recorded from the initial sale date (or first available date) through the event date, yielding a comprehensive view of the pricing lifecycle. We denote the raw time-series data as
\[
\mathcal{D}_{\text{raw}} = \{\mathbf{x}_t\}_{t=1}^T,
\]
where $\mathbf{x}_t = [x_t^{(1)}, x_t^{(2)}, \dots, x_t^{(d)}]^\top $ is a vector of $d$ observed 
variables at time $t$, and $T$ is the total number of recorded time steps. Variables include artist, 
event date/time, venue, price collection date/time, mean price, median price, low price, high price, 
and listing count. This can equivalently be represented as a matrix:
\[
\mathbf{X} \in \mathbb{R}^{T \times d},
\]
with rows corresponding to time steps and columns to variables. Figure~\ref{fig:time_series_example} 
illustrates this representation using ticket price data for blues guitarist Buddy Guy at the Wilbur 
Theatre in Boston on 10/3/2023.

\begin{figure}[t]
  \centering
  \hspace*{\fill}
  \begin{subfigure}[c]{0.48\columnwidth}
      \includegraphics[width=\linewidth]{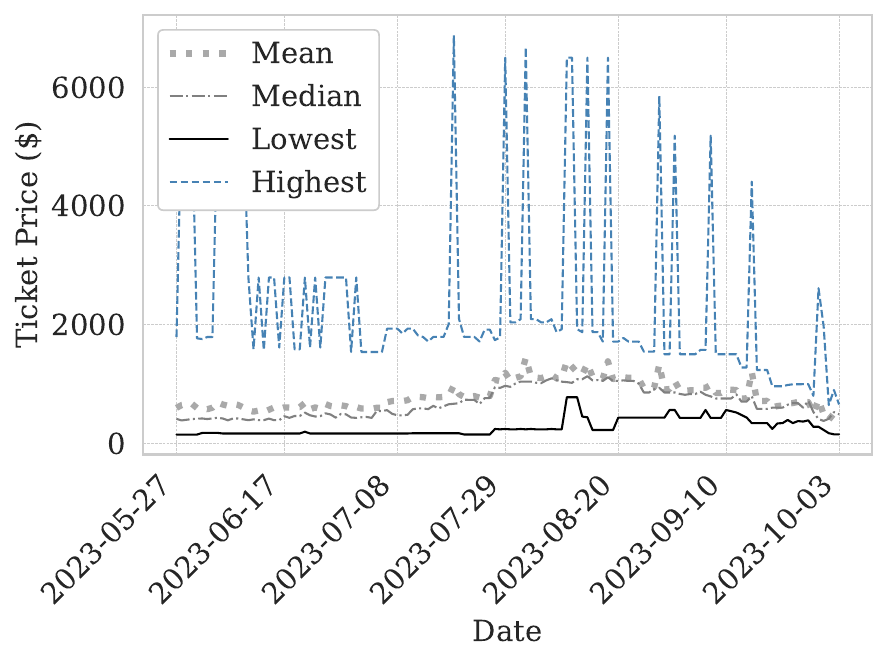}
      \caption{Ticket prices over time for Buddy Guy at Wilbur Theatre, Boston, MA, 10/3/2023, showing 
      the Mean, Median, Lowest, and Highest prices.\\
      \\
      \href{https://seatgeek.com}{\includegraphics[width=.1\textwidth]{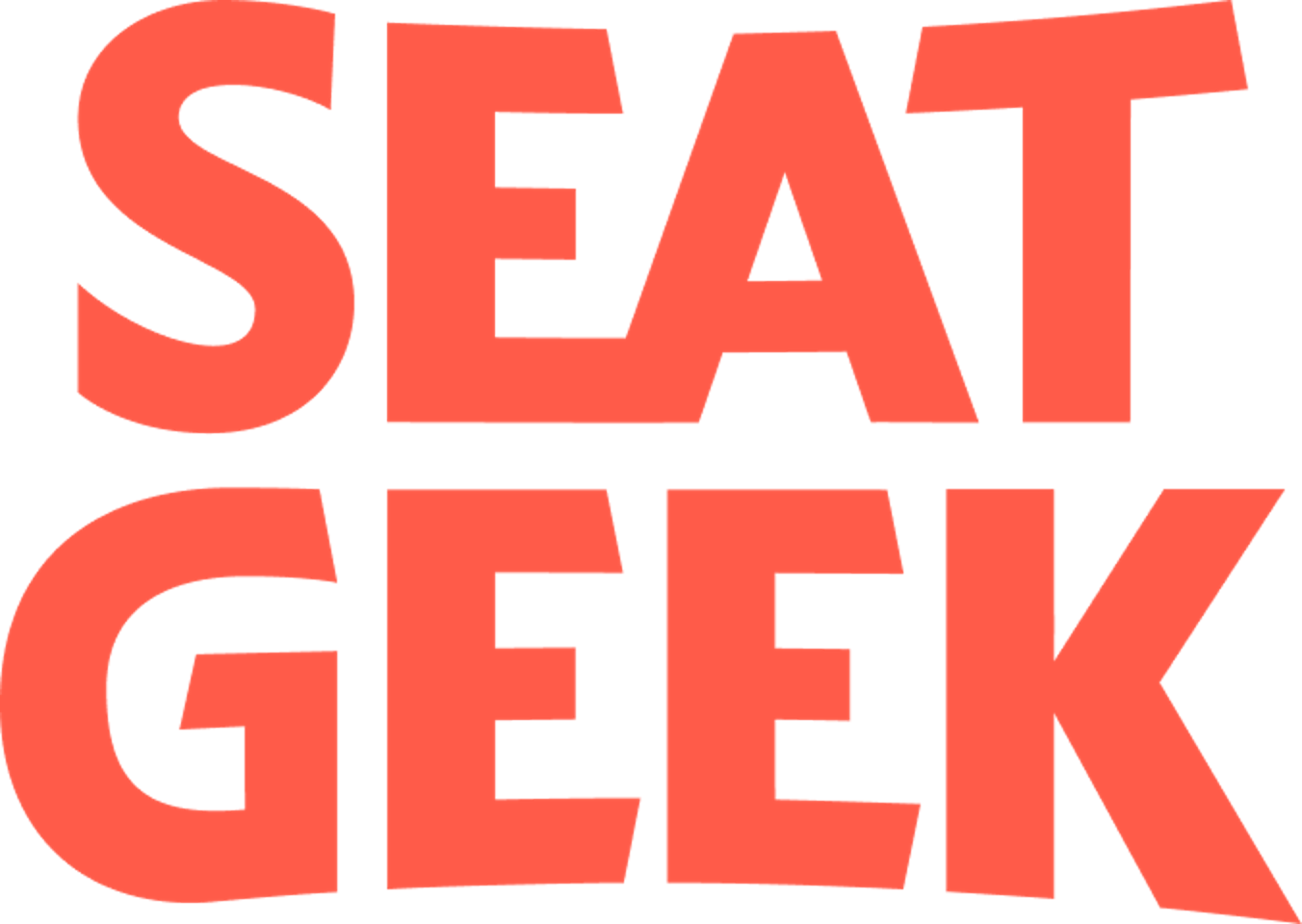}}
      }
      \label{fig:time_series_example}
  \end{subfigure}
  \hspace*{\fill}
  \begin{subfigure}[c]{0.48\columnwidth}
      \includegraphics[width=\linewidth]{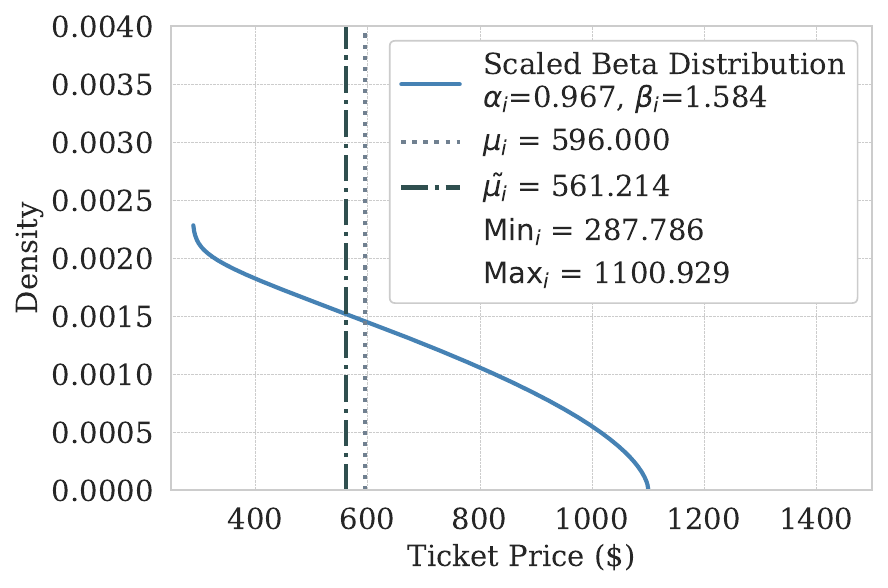}
      \caption{Estimated scaled beta distribution for Buddy Guy at Wilbur Theatre, Boston, MA, 10/3/2023. 
      The figure shows the estimated $\alpha_i$ and $\beta_i$ parameters, the mean price ($\mu_i$), 
      the median price ($\tilde{\mu}_i$), the lowest price ($\text{Min}_i$), and the highest price ($\text{Max}_i$). 
      These quantities define the scaled beta distribution at snapshot $T'$, leading up to event $i$ on 10/3/2023. 
      They represent the economic signature and corresponding feature value set for this event in the Random Forest model.}
      \label{fig:beta_distribution_plot}
  \end{subfigure}
  \hspace*{\fill}
  \caption{Event Overview, Buddy Guy at Wilbur Theatre, Boston, MA, 10/3/2023}
  \label{fig:event_overview}
\end{figure}

\subsection{Representations for classification}
\label{sec:data_representations}
To prepare the data for artist classification, we define a transformation $f$ from the raw time-series 
data into a structured feature space:
\[
f: \mathbb{R}^{T \times d} \rightarrow \mathbb{R}^{E \times n}.
\]
The resulting training dataset is given by:
\[
\mathcal{D}_{\text{train}} = f(\mathcal{D}_{\text{raw}}) = \{(\mathbf{z}_i,y_i)\}_{i=1}^E,
\]
where $\mathbf{z}_i \in \mathbb{R}^n$ is the derived feature vector for the $i$-th event, and $y_i$ 
is the corresponding artist label. This structured format is amenable to standard machine learning 
methods.

We frame artist classification as pairwise binary tasks (e.g., The Pixies vs. Billy Joel), using each model both for identification and as a test of whether artist-specific pricing distributions are separable. High accuracy signals a faithful representation of underlying distributional dynamics, while lower accuracy warrants further study. Random Forests suit this problem because they capture complex feature interactions and are robust to noise, and feature-based time-series classification shows that summary statistics can be strongly discriminative \citep{middlehurst2024bakeoff}. Toolkits like TSFresh \citep{christ2018tsfresh} and Catch22 \citep{lubba2019catch22} achieve high accuracy with handcrafted statistics, and ensemble methods such as Canonical Interval Forest (CIF) further improve performance via interval-based features \citep{middlehurst2024bakeoff}. Building on these insights, we summarize pricing time series into compact distributional representations that capture artist-specific patterns in pricing dynamics.

\topicstart{Basic statistical features:}
Initially, we derive basic summary statistics for the subsequence $T'$ leading up to each event:
\begin{equation}
x = \frac{1}{|T'|} \sum_{t \in T'} x_t,\quad x_t \in \{\mu_t, \tilde{\mu}_t, \text{Max}_t, \text{Min}_t\},
\end{equation}
where $\mu_t, \tilde{\mu}_t, \text{Max}_t, \text{Min}_t$ denote mean, median, maximum, and minimum 
prices respectively. This yields the feature vector and dataset:
\begin{equation}
\mathbf{z}_i = [\mu_i,\; \tilde{\mu}_i,\; \text{Max}_i,\; \text{Min}_i]^\top \qquad \mathcal{D}_{\text{basic}} = \{(\mathbf{z}_i,y_i)\}_{i=1}^E.
\end{equation}

\topicstart{Distribution-augmented features:}
Basic statistics alone omit nuanced distributional shapes. To address this, we estimate scaled beta 
distribution parameters $\alpha_i, \beta_i$ for each event-artist pair over 
$[\text{Min}_i, \text{Max}_i]$ (Section \ref{sec:estimation}). These parameters enrich the feature vector and yield the dataset:
\begin{equation}
\mathbf{z}_i = [\mu_i,\; \tilde{\mu}_i,\; \text{Max}_i,\; \text{Min}_i,\; 
\alpha_i,\; \beta_i]^\top \qquad \mathcal{D}_{\alpha\beta} = \{(\mathbf{z}_i,y_i)\}_{i=1}^E.
\end{equation}

\subsection{Implicit regularization via zero-variance features}
To implicitly regularize our Random Forest models, we augment dataset $\mathcal{D}_{\alpha\beta}$ with 
$n_{\mathrm{ZV}}$ zero-variance (constant-value) features $\mathbf{c}\in\mathbb{R}^{n_{\mathrm{ZV}}}$:
yielding the updated feature vector and dataset:
\begin{equation}
\mathbf{z}_i = [\mathbf{z}_i,\; \mathbf{c}]^\top \qquad \mathcal{D}_{\alpha\beta}^{(\text{reg})} = \{(\mathbf{z}_i,y_i)\}_{i=1}^E.
\end{equation}

Although counterintuitive, constant-value features subtly shift Random Forest feature-selection 
probabilities, implicitly promoting deeper, more robust trees and improved generalization, as explored 
in detail in Section \ref{sec:regularization}.

\topicstart{Additional validation (handwritten digits):}
To verify that implicit regularization effects generalize beyond ticket pricing data, we replicate our 
approach using the standard UCI handwritten digits dataset \citep{uci_digits}. Specifically, we form two 
analogous datasets: $\mathcal{D}_\delta$, containing the original digit features, and 
$\mathcal{D}_{\delta}^{(\text{reg})}$, which includes additional zero-variance features to mirror the 
ticket pricing methodology. This parallel validation confirms (Section \ref{sec:regularization}) the consistency and 
generalizability of the observed regularization effects across distinct data domains.

\section{Distribution recovery from limited statistics}
\label{sec:estimation}

Concert ticket price distributions are modeled for each event-artist pair using a scaled beta distribution. We estimate $\alpha_i$ and $\beta_i$ via composite quantile and moment matching from the SeatGeek API minimum, maximum, mean, and median. Scaled beta flexibly captures bounded shapes and provides a nuanced snapshot of pricing dynamics. Unlike \citep{wei2024latent}, which matches multiple quantiles in a latent space, and \citep{zhang2023quantile}, which parameterizes flows via quantiles, we fit $\alpha$ and $\beta$ using only the median and mean, an effective summary-based strategy on minimal data. \citep{dempster2024quant} also show the value of quantile features, though they require richer raw coverage than SeatGeek. Classical formulas in \citep{krishnamoorthy2016handbook} provide the beta moments and parameter relations we use. 

Two complementary validations are presented. Sections~\ref{sec:theoretical_bounds} and~\ref{sec:random_forest_results} test fidelity indirectly via out-of-sample classification accuracy consistent with the accuracy–fidelity theory. Section~\ref{sec:synthetic_validation} directly verifies the theory in a controlled synthetic pipeline: generate a scaled beta law $P_{\theta}$ with $\theta=(\alpha_i,\beta_i)$, compute $\mathrm{Min}_i,\mathrm{Max}_i,\mu_i,\tilde\mu_i$, reconstruct $\hat\theta=(\alpha'_i,\beta'_i)$ from $\mu_i,\tilde\mu_i$ after standardizing by $(\mathrm{Min}_i,\mathrm{Max}_i)$, then measure relationships among logistic loss, $\mathrm{TV}(P_{\hat\theta},P_{\theta})$, and $\mathrm{JS}(P_{\hat\theta},P_{\theta})$. Loss to divergence curves fall within $\tfrac{1}{2}\,\mathrm{TV}^2 \le \mathrm{JS} \le (\ln 2)\,\mathrm{TV}$ and align with the Lipschitz margin arguments, supporting accuracy as a practical surrogate for fidelity when only $\{\mathrm{Min}_i,\mathrm{Max}_i,\mu_i,\tilde\mu_i\}$ are available.

\subsection{Scaled beta distribution}
\label{sec:scaled_beta_distribution}

With variables 
$\text{Min}_i$, $\text{Max}_i$, $\mu_i$, and $\tilde{\mu}_i$ defined in Section \ref{sec:preliminaries}, the
probability density function (PDF) of the scaled beta distribution is given by:

\begin{equation}
f(x; \alpha_i, \beta_i, \text{Min}_i, \text{Max}_i) = 
\frac{(x - \text{Min}_i)^{\alpha_i - 1}\,(\text{Max}_i - x)^{\beta_i - 1}}
{(\text{Max}_i - \text{Min}_i)^{\alpha_i + \beta_i - 1}\,B(\alpha_i, \beta_i)},
\end{equation}
where $x$ is the ticket price, and $B(\alpha_i, \beta_i)$ is the beta function. This
formulation transforms the standard beta distribution from $[0,1]$ to
$[\text{Min}_i,\text{Max}_i]$. Such a scaled beta framework is also seen in other contexts, like
\citep{zhou2023beta}, who exploit beta distributions for bounded data in
generative modeling, underscoring the flexibility of beta-based parameterizations. Classical
discussions in \citep{krishnamoorthy2016handbook} elaborate on these beta
formulations and offer general moment-based inference approaches that set the stage for our
scaled version.

\subsection{Closed-form estimation: composite quantiles and moments}
\label{sec:quantile_moment_matching}

To estimate the parameters $\alpha_i$ and $\beta_i$, we reparameterize the scaled beta
distribution using the $\mu_i$ and $\tilde{\mu}_i$ provided by the SeatGeek API. 
The $\mu_i$ and $\tilde{\mu}_i$ for the beta distribution on $[\text{Min}_i,\text{Max}_i]$ are given by:
\begin{equation}
\mu_i \;=\; \text{Min}_i \;+\; \frac{\alpha_i}{\alpha_i + \beta_i}\,
(\text{Max}_i - \text{Min}_i),
\end{equation}
\begin{equation}
\tilde{\mu}_i \;\approx\; \text{Min}_i \;+\; (\text{Max}_i - \text{Min}_i)
\Bigl(\frac{\alpha_i - \tfrac13}{\,\alpha_i + \beta_i - \tfrac23\,}\Bigr).
\end{equation}
We first scale the mean and median to $[0,1]$:
\begin{equation}
s \;=\; \frac{\mu_i - \text{Min}_i}{\text{Max}_i - \text{Min}_i}, 
\quad q \;=\; \frac{\tilde{\mu}_i - \text{Min}_i}{\text{Max}_i - \text{Min}_i}.
\end{equation}
From the mean equation, we express $\beta_i$ in terms of $\alpha_i$ and $s$, and substitute into the median equation to obtain:
\begin{equation}
\beta_i \;=\; \alpha_i \Bigl(\frac{1 - s}{s}\Bigr) \quad\text{and}\quad q \;=\; \frac{\alpha_i - \tfrac13}{\,\tfrac{\alpha_i}{s} - \tfrac23\,}.
\end{equation}
This simplifies to:
\begin{equation}
\alpha_i \;=\; \frac{\,s\,(2q - 1)}{\,3\,(q - s)}, 
\quad \beta_i \;=\; \frac{\,(\,1 - s)\,(2q - 1)}{\,3\,(q - s)}.
\end{equation}

This method estimates the underlying price distribution from minimal statistics, capturing central tendency and shape, which improves predictive performance in downstream models. \citep{salimans2024moment} show that matching selected moments can preserve generative behavior, supporting our use of a mean plus a single quantile (the median) to influence the inferred distribution. While \citep{dempster2024quant} use many raw-data quantiles, our composite quantile-and-moment matching uses only SeatGeek API summaries. Canonical beta identities from \citep{krishnamoorthy2016handbook} validate fitting $\alpha$ and $\beta$ from so few statistics.

\citep{wei2024latent} show that matching multiple quantiles in latent space can further align distributions. Our approach is simpler, fitting scaled beta parameters directly in the observable ticket-price space using only the mean and median, yet it demonstrates how a small, well-chosen set of statistics yields useful distributional insight. \citep{lubba2019catch22} likewise find that compact feature sets can preserve classification strength, supporting our reliance on $\{\mu_i,\,\tilde{\mu}_i,\,\alpha_i,\,\beta_i\}$ with $\text{Min}_i$ and $\text{Max}_i$. An example snapshot appears in Fig.~\ref{fig:beta_distribution_plot}.

In terms of efficiency, the estimator runs in constant time $O(1)$ per event with a few arithmetic operations. Root-finding for $(\alpha,\beta)$ from mean and median typically needs $O(I)$ special-function evaluations with $I\approx 5\text{ to }20$ iterations. Constrained optimizers incur $O(I\,C_{\mathrm{grad}})$ from gradient computations and line search. Grid search costs $O(G)$ for a grid of size $G$ unless supported by precomputation and interpolation. Simulation-based or Bayesian methods are more general but scale as $O(N_{\mathrm{iter}})$ with larger constants. Our closed-form solution is the most efficient among these options and is well suited to large SeatGeek-scale applications.

\subsection{Kernel density estimation for distributional features}

Given the derivations for $\alpha_i$ and $\beta_i$ alongside the original statistical features, we compare these components across events for specific acts to identify where $\alpha_i$ and $\beta_i$ add predictive power. Larger distances between the feature distributions of two acts indicate greater separability. Consider two acts $\{1,2\}$ in a pairwise setting.

Formally, for a given act, define
\[
\mathbf{z}_i 
= \bigl[\mu_i,\;\tilde{\mu}_i,\;\text{Max}_i,\;
\text{Min}_i,\;\alpha_i,\;\beta_i\bigr]^\top
\]
for each event $i$. Let $x \in \{\mu, \tilde{\mu}, \text{Max}, \text{Min}, \alpha, \beta\}$. The kernel density estimate (KDE) for each feature is
\begin{equation}
\hat{f}_{x}(x)
= \frac{1}{E\,h}\sum_{i=1}^{E}K\Bigl(\frac{x - x_i}{h}\Bigr),
\end{equation}
where $K$ is the kernel, $E$ is the number of events, and $h$ is the bandwidth. Using the KDE for each feature and act, $\{\hat{f}^{act}_{\mu}, \hat{f}^{act}_{\tilde{\mu}}, \hat{f}^{act}_{\text{Max}}, 
\hat{f}^{act}_{\text{Min}}, \hat{f}^{act}_{\alpha}, \hat{f}^{act}_{\beta}\}$, we assess distributional similarity with Hellinger distance $H(\hat{f}^{1}_{x}, \hat{f}^{2}_{x})$ and Jensen-Shannon divergence $JS(\hat{f}^{1}_{x} \parallel \hat{f}^{2}_{x})$:
\begin{enumerate}
  \item Hellinger Distance:
  \begin{equation}
  H(\hat{f}^{1}_{x}, \hat{f}^{2}_{x})
  = 
  \frac{1}{\sqrt{2}}\sqrt{ \int
  \bigl(\sqrt{\hat{f}^{1}_{x}(t)} - \sqrt{\hat{f}^{2}_{x}(t)}\bigr)^2 dt }.
  \end{equation}
  \item Jensen-Shannon Divergence:
  \begin{equation}
  JS(\hat{f}^{1}_{x}\parallel \hat{f}^{2}_{x})
  = 
  \tfrac12\,KL(\hat{f}^{1}_{x}\parallel M)
  + \tfrac12\,KL(\hat{f}^{2}_{x}\parallel M),
  \end{equation}
  where $M = \tfrac12(\hat{f}^{1}_{x} + \hat{f}^{2}_{x})$ and
  \begin{equation}
  KL(\hat{f}^{1}_{x} \parallel \hat{f}^{2}_{x}) 
  = \int \hat{f}^{1}_{x}(t)\,\log\Bigl(\frac{\hat{f}^{1}_{x}(t)}{\hat{f}^{2}_{x}(t)}\Bigr)\,dt.
  \end{equation}
\end{enumerate}

These distances score each feature’s ability to separate acts. The estimated parameters $\alpha_i$ and $\beta_i$ often improve Random Forest accuracy by sharpening separability. For example, in Fig.~\ref{fig:kde_alpha_beta_params} comparing Drake and Olivia Rodrigo, the KDEs show $\alpha_i$ is more distinctive than the original features, which is reflected in both Hellinger and JS.

While \citep{dempster2024quant} extract many quantiles from raw data, our setting uses summary statistics to compute $\alpha_i$ and $\beta_i$. Krishnamoorthy’s discussion~\citep{krishnamoorthy2016handbook} emphasizes how beta shape parameters capture subtle differences. Here those shape and skew measures both aid classification and are validated by it.

Our use of the Jensen-Shannon distance here differs from Sect.~\ref{sec:theoretical_bounds}, where it supports formal bounds on convergence. In this subsection it is an empirical tool for comparing feature-density profiles across artists.

\begin{figure}[t]
    \centering
    \includegraphics[width=\textwidth]{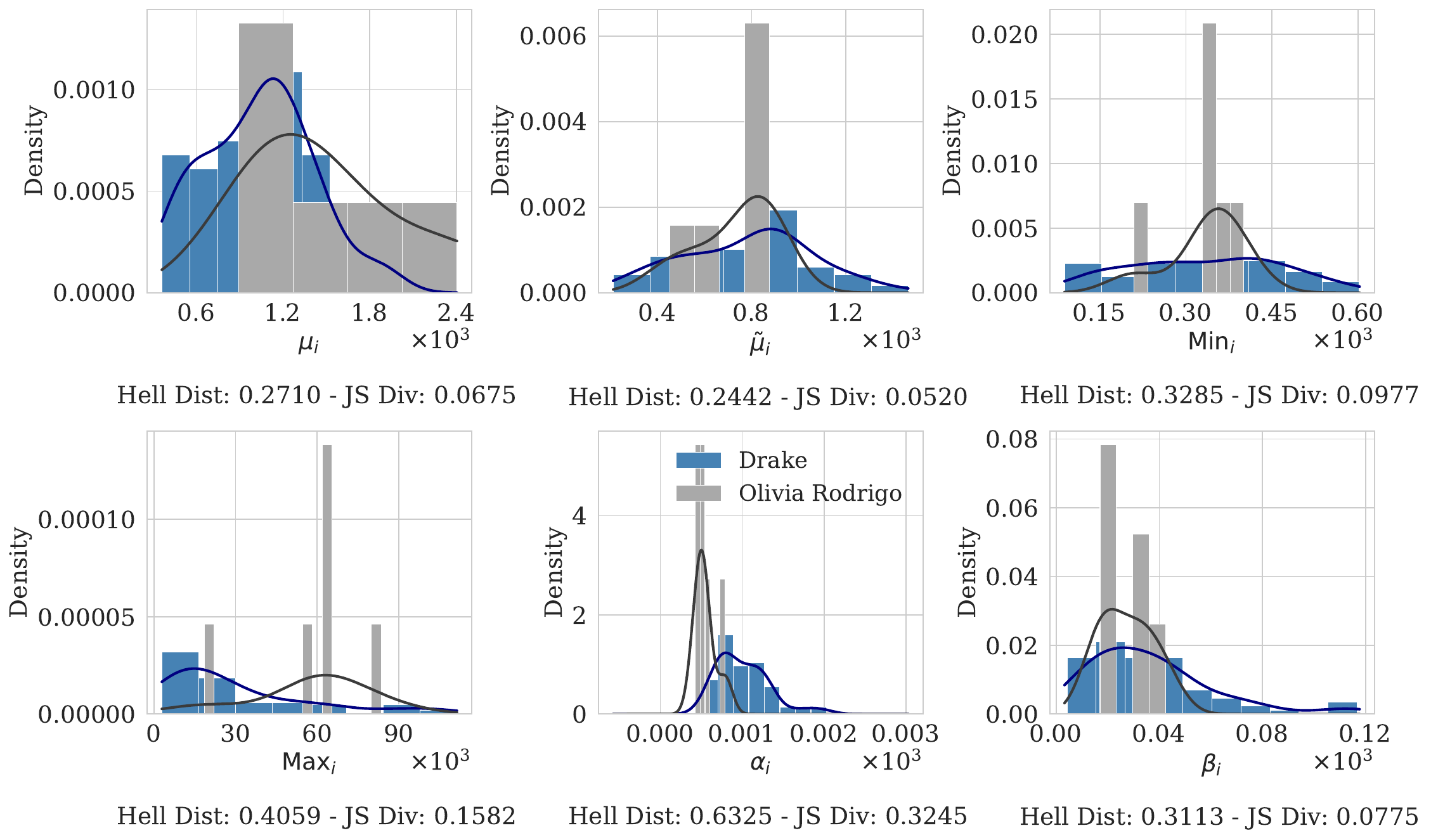}
    \caption{The plots show the distributions of each feature across all events for artists
      Drake and Olivia Rodrigo. The Hellinger distance and Jensen-Shannon divergence are calculated between
      each distribution. In this particular comparison of artists, the $\alpha_i$ 
      parameter offers the most distinctive density profile across all events, as indicated by the
      distribution distance metrics.}
      \label{fig:kde_alpha_beta_params}
    \vskip -0.2in
\end{figure}

\subsection{Validation via classification accuracy}
\label{sec:theoretical_bounds}

We justify classification as a validation tool by linking classification accuracy to parameter estimation accuracy. If the estimated parameters $\hat{\theta}$ are close to the true parameters $\theta$ then classification performance should rise. Conversely, high classification performance provides empirical evidence that the estimates capture the underlying distribution, so accuracy can serve as a measurable proxy for distributional correctness when ground truth is unavailable. This builds on Tsybakov's margin assumption \citep{tsybakov2004optimal} and the probabilistic bounds of Devroye et al. \citep{devroye1996probabilistic}. We extend these results by characterizing the connection between accuracy and distributional similarity through total variation distance and Jensen-Shannon divergence \citep{lin1991divergence}, showing that improvements in accuracy yield stronger guarantees of distributional reliability with quadratic convergence in the information-theoretic setting.

\begin{proposition}[Parameter Estimation Consistency via Classification Accuracy]
Let $\Theta \subset \mathbb{R}^d$ be the space of parameters, where each probability
distribution $P$ is parameterized by $\theta \in \Theta$. Define a feature map
\begin{equation}
\phi(\theta) 
= \Bigl(f_1(P), f_2(P), \dots, f_{k-d}(P), \theta \Bigr),
\end{equation}
where $f_i(P)$ represents summary statistics of $P$, such as $\text{Min}_i$, $\text{Max}_i$, $\tilde{\mu}_i$, and $\mu_i$. 
A classifier $f : \mathbb{R}^k \to \{0,1\}$ is trained to distinguish between two classes
based on $\phi(\hat{\theta})$, where $\hat{\theta}$ is an estimated parameter obtained from
observed data. If $f$ achieves a classification error rate $\varepsilon$, then there exists a function
$\delta(\varepsilon) \to 0$ as $\varepsilon \to 0$ such that the estimation error satisfies
\begin{equation}
\|\hat{\theta} - \theta\| \;\le\; \delta(\varepsilon).
\end{equation}
\end{proposition}

\begin{proof}
\topicstart{Propagation of estimation error to feature space:} 
Define the ``true'' feature vector by $X^\ast = \phi(\theta)$ and the observed feature vector
by $X = \phi(\hat{\theta})$. By the Lipschitz condition,
\begin{equation}
\|\,X - X^\ast\| 
= \|\phi(\hat{\theta}) - \phi(\theta)\|
\;\le\; L\,\|\hat{\theta} - \theta\|.
\end{equation}

\topicstart{Relating feature perturbation to classification error:}
Under the margin separation assumption, the ideal feature vectors for two classes are separated
by at least $\Delta$. Suppose that a perturbation of size $\gamma$ in feature space is
tolerable without altering class assignment. Then,
\begin{equation}
R(f) \;\ge\; \mathbb{P}\bigl(\|\,X - X^\ast\| \;\ge\; \tfrac{\Delta}{2}\bigr).
\end{equation}

\topicstart{Bounding the probability of large feature perturbation:}
Using Markov's inequality,
\begin{equation}
\mathbb{P}\bigl(\|\,X - X^\ast\| \;\ge\; \tfrac{\Delta}{2}\bigr) 
\;\le\; \frac{2}{\Delta}\,\mathbb{E}\bigl[\|\,X - X^\ast\|\bigr].
\end{equation}
Combining with the Lipschitz bound,
\begin{equation}
\mathbb{P}(E) 
\;\le\; \tfrac{2}{\Delta}\,L\,\|\hat{\theta} - \theta\|.
\end{equation}
Then using the risk bound $R(f) \le \varepsilon$, we obtain:
\begin{equation}
\varepsilon \;\ge\; \tfrac{2L}{\Delta}\,\|\hat{\theta} - \theta\|.
\end{equation}
Rearranging,
\begin{equation}
\|\hat{\theta} - \theta\| \;\le\; \tfrac{\Delta}{2L}\,\varepsilon.
\end{equation}
Setting $\delta(\varepsilon) = \tfrac{\Delta}{2L}\,\varepsilon$, we see $\delta(\varepsilon)
\to 0$ as $\varepsilon \to 0$, proving the proposition.
\end{proof}

Building on this foundation, we can more precisely characterize the relationship between classification 
accuracy and distributional similarity. The following theorems extend our theoretical analysis to establish 
rigorous bounds between classifier error and common measures of distributional difference.

\begin{theorem}[Classification Accuracy and Total Variation Distance]
  \label{thm:total_variation}
Let $P_{\hat\theta}$ and $P_{\theta}$ denote distributions on a common sample space $\mathcal{X}$ with densities $p_{\hat\theta}(x)$ and $p_{\theta}(x)$, parameterized by the estimated parameters $\hat\theta$ and true parameters $\theta$, respectively. Let $\varepsilon$ be the misclassification error probability of a classifier built upon $P_{\hat\theta}$. Then the total variation distance between the distributions is bounded by:
\[
TV(P_{\hat\theta}, P_{\theta}) \;=\; \tfrac{1}{2}\!\int_{\mathcal{X}} \big|\,p_{\hat\theta}(x)-p_{\theta}(x)\,\big|\,dx \;\le\; \eta(\varepsilon),
\]
\begin{equation}
\text{where} \quad \eta(\varepsilon)\rightarrow 0 \quad \text{as} \quad \varepsilon\rightarrow 0.
\end{equation}
\end{theorem}

\begin{proof}
\topicstart{Misclassification and distributional differences:}
Consider a binary classifier with decision regions $C_{\hat{\theta}}$ and $C_{\theta}$ corresponding 
to estimated and true parameters, respectively, and assume equal class priors. The misclassification 
probability $\varepsilon$ is given by:
\begin{equation}
\varepsilon \;=\; \frac{1}{2}\int_{\mathcal{X}} \!\big[\,p_{\hat{\theta}}(x)\,I(x\in C_{\theta}) \;+\; 
p_{\theta}(x)\,I(x\in C_{\hat{\theta}})\,\big]\,dx,
\end{equation}
where $I(\cdot)$ is the indicator function. We compare $\varepsilon$ to the Bayes-optimal classification error $\varepsilon^*$, given explicitly by the total variation distance \citep{devroye1996probabilistic}:
\begin{equation}
\varepsilon^* \;=\; \frac{1}{2}\left[1 - TV(P_{\hat{\theta}}, P_{\theta})\right] 
\;=\; 
\frac{1}{2}\left[1 - \frac{1}{2}\int_{\mathcal{X}} \big|\,p_{\hat{\theta}}(x)-p_{\theta}(x)\,\big|\,dx\right].
\end{equation}
Since the achieved error $\varepsilon$ must exceed the Bayes-optimal error $\varepsilon^*$, we have:
\begin{equation}
\varepsilon \;\geq\; \varepsilon^* \;=\; \frac{1}{2}\left[1 - TV(P_{\hat{\theta}}, P_{\theta})\right].
\end{equation}
Rearranging terms explicitly isolates the total variation distance:
\begin{equation}
TV(P_{\hat{\theta}}, P_{\theta}) \ \geq\ 1 - 2\varepsilon.
\end{equation}
In particular, $TV(P_{\hat{\theta}}, P_{\theta}) \ge \max\{0,\,1-2\varepsilon\}$, so the bound is vacuous only when $1-2\varepsilon<0$. This provides a fundamental lower bound linking classification error and distributional differences. However, we also seek a meaningful upper bound.

\topicstart{Upper bound via parameter continuity:}
From the previous proposition, we have a direct parameter-based bound:
\[
\|\hat{\theta}-\theta\|\leq \delta(\varepsilon), 
\quad\text{with}\quad\delta(\varepsilon)\rightarrow 0\quad\text{as}\quad\varepsilon\rightarrow 0.
\]
Assume the parametric family $\{P_{\theta}\}$ is Lipschitz-continuous in parameters in total variation, 
meaning there exists a constant $L>0$ such that:
\begin{equation}
TV(P_{\hat{\theta}}, P_{\theta}) \;\leq\; \frac{L}{2}\,\|\hat{\theta}-\theta\|.
\end{equation}
This condition typically holds for parametric distributions like the scaled beta considered 
in this work, where densities vary smoothly with respect to parameters. Substituting the result from the Proposition, we get:
\begin{equation}
TV(P_{\hat{\theta}}, P_{\theta}) \;\leq\; \frac{L}{2}\,\delta(\varepsilon).
\end{equation}
Define $\eta(\varepsilon)=\frac{L}{2}\,\delta(\varepsilon)$, which clearly approaches zero as 
$\varepsilon\rightarrow 0$. Thus, we have established a rigorous upper bound directly relating classifier error to total 
variation distance:
\[
TV(P_{\hat{\theta}}, P_{\theta}) \;\leq\; \eta(\varepsilon),\quad 
\eta(\varepsilon)\rightarrow 0\quad\text{as}\quad\varepsilon\rightarrow 0.
\]
\end{proof}

While the Total Variation distance provides a natural measure of distributional difference, 
information-theoretic measures can offer additional insights with stronger convergence properties. 
The following theorem establishes an even more precise relationship using the Jensen-Shannon divergence.

\begin{theorem}[Classification Accuracy and Jensen-Shannon Divergence]
  \label{thm:jensen_shannon}
Under the same conditions as the previous theorem, the Jensen-Shannon divergence between the 
distributions can be bounded by:
\begin{equation}
JS(P_{\hat{\theta}}||P_{\theta}) \leq \xi(\varepsilon),
\quad\text{where}\quad \xi(\varepsilon)\rightarrow 0 \ \text{as}\ \varepsilon\rightarrow 0.
\end{equation}
Furthermore, in the small-error regime (i.e., for sufficiently small $\varepsilon$) and under mild regularity, this bound exhibits a quadratic convergence rate.
\end{theorem}

\begin{proof}
\topicstart{Relationship between JS divergence and total variation distance:}
The Jensen-Shannon divergence between distributions $P_{\hat{\theta}}$ and $P_{\theta}$ is 
defined as:
\begin{equation}
JS(P_{\hat{\theta}}||P_{\theta}) = \frac{1}{2}KL(P_{\hat{\theta}}||M) + 
\frac{1}{2}KL(P_{\theta}||M),
\end{equation}
where $M = \frac{1}{2}(P_{\hat{\theta}} + P_{\theta})$ is the mixture distribution, and $KL$ 
is the Kullback–Leibler divergence. We recall (see, e.g., \citep{lin1991divergence}) that, globally, 
$JS$ is Lipschitz in total variation:
\begin{equation}
\label{eq:js-linear-tv}
JS(P_{\hat{\theta}}||P_{\theta}) \ \le\ (\ln 2)\,TV(P_{\hat{\theta}}, P_{\theta}),
\end{equation}
where $TV(P_{\hat{\theta}}, P_{\theta})$ is the total variation distance as defined in the 
previous theorem.

\topicstart{Applying the total variation bound from the previous theorem:}
From the previous theorem, we have established that
\[
TV(P_{\hat{\theta}}, P_{\theta}) \leq \eta(\varepsilon) = \frac{L}{2}\,\delta(\varepsilon).
\]
Substituting this into \eqref{eq:js-linear-tv} yields the global vanishing bound
\begin{equation}
\label{eq:global-linear}
JS(P_{\hat{\theta}}||P_{\theta})
\ \le\ 
(\ln 2)\,\eta(\varepsilon)
\ =\
(\ln 2)\,\frac{L}{2}\,\delta(\varepsilon)
\ :=\ 
\xi(\varepsilon),
\end{equation}
so that $\xi(\varepsilon)\to 0$ as $\varepsilon\to 0$.

\topicstart{Quadratic convergence in the small-error regime:}
Moreover, when the discrepancy is small, $JS$ admits a second-order (quadratic) control in TV under mild regularity (e.g., the relevant densities are bounded away from $0$ and $\infty$, or the likelihood ratio is bounded). Thus there exist constants $C>0$ and $\tau>0$ (depending only on those regularity parameters) such that
\begin{equation}
\label{eq:local-quadratic}
\text{if}\quad TV(P_{\hat{\theta}},P_{\theta}) \le \tau
\quad\text{then}\quad
JS(P_{\hat{\theta}}||P_{\theta}) \ \le\ C\,TV^2(P_{\hat{\theta}},P_{\theta}).
\end{equation}
Combining \eqref{eq:local-quadratic} with $TV(P_{\hat{\theta}}, P_{\theta}) \le \eta(\varepsilon)$ from above gives, whenever $\eta(\varepsilon)\le \tau$,
\begin{equation}
JS(P_{\hat{\theta}}||P_{\theta})
\ \le\ 
C\,\eta^2(\varepsilon)
\ =\
C\left(\frac{L}{2}\,\delta(\varepsilon)\right)^2
\ :=\
\tilde\xi(\varepsilon),
\end{equation}
which shows a quadratic convergence rate in the small-error regime since $\delta(\varepsilon)\to 0$ as $\varepsilon\to 0$.
\end{proof}

The progression from total variation distance to Jensen-Shannon divergence reveals a finer relationship: globally linear in error and quadratic in the small-error regime under mild regularity. For ticket pricing this means that as classification accuracy improves, i.e., as $\varepsilon$ decreases, the estimated beta parameters approach the true parameters at least linearly, with an accelerating quadratic rate once in the small-error regime. The Jensen-Shannon divergence offers advantages over total variation:
(1) tighter small-error convergence via the quadratic relationship while retaining a global linear guarantee,
(2) a natural information-theoretic view of distinguishability,
(3) bounded range $[0,\log 2]$ or $[0,1]$ in bits,
(4) symmetry unlike KL. 

In our ticket pricing context, modest gains in classifier accuracy produce increasingly large improvements in agreement between estimated and true beta parameters, especially once accuracy is high. This supports parameter estimation from limited statistics and strengthens the theoretical basis for using the estimated $\alpha_i$ and $\beta_i$ in downstream tasks. By relating classification error to both total variation and Jensen-Shannon divergence, these results connect practical machine-learning performance with rigorous statistical inference and tie learning theory to our modeling objectives.

\topicstart{Application to ticket pricing and artist classification:}
In our setting, $P$ is a scaled beta distribution for ticket prices with parameters $\theta=(\alpha_i,\beta_i)$. The summary statistics $f_i(P)$ are $\text{Min}_i$, $\text{Max}_i$, $\mu_i$, and $\tilde{\mu}_i$, and the classifier $f$ distinguishes artists using pricing information. According to the theorem, artist classification accuracy validates the estimated parameters $\hat{\alpha}_i$ and $\hat{\beta}_i$. Strong classifier performance implies that the reconstructed beta distribution is close to the true pricing distribution. Empirical results in the next section confirm this, showing that including $\alpha_i$ and $\beta_i$ in the feature set improves accuracy. This aligns with \citep{tsybakov2004optimal,devroye1996probabilistic}, which link precise estimation to better prediction, and with \citep{krishnamoorthy2016handbook}, which notes that accurate beta inference can rely on a few well chosen statistics. Tight approximation of $\alpha_i$ and $\beta_i$ produces measurable gains in classification accuracy.

\subsection{Random Forest results}
\label{sec:random_forest_results}

Random Forests build multiple decision trees and aggregate their predictions to improve generalization \citep{ho1998rsm, breiman2001random}. For input $x$,
\begin{equation}
\hat{y} = \frac{1}{B}\,\sum_{b=1}^{B} h_b(x),
\end{equation}
where $h_b(x)$ is the $b$th tree and $B$ is the number of trees. Trees train on bootstrap samples and use random feature subsets at each split, which reduces variance and limits overfitting relative to single trees. We use the standard scikit-learn implementation \citep{scikit-learn}.

In our classification task we identify the artist from pricing information. Each Random Forest is trained on a pair of artists. This dyadic setup mitigates class imbalance compared with one versus all classification.

\topicstart{Empirical classification improvements with parameter estimates:}
To assess the effect of including $\alpha_i$ and $\beta_i$ in the feature set, we performed a pairwise comparison of Random Forest classifiers trained on two sets. The first, $\mathcal{D}_{\text{basic}}$, contains $\mu_i$, $\tilde{\mu}_i$, $\text{Min}_i$, and $\text{Max}_i$ as in Sect.~\ref{sec:preliminaries}. The second, $\mathcal{D}_{\alpha \beta}$, augments these with estimated $\alpha_i$ and $\beta_i$ from the scaled beta distribution. The task is binary artist classification with target $y_i$. We compare overall accuracy and count how many models improve when using $\mathcal{D}_{\alpha \beta}$. We used $N_{\text{pair}} = 20{,}000$ paired observations of the same artist classification problem solved with both $\mathcal{D}_{\text{basic}}$ and $\mathcal{D}_{\alpha \beta}$, spanning $N_{\text{artist}} = 954$ unique artists. Each pair has $N^{\text{train}}_{\text{event}} \approx 37$ training events and $N^{\text{test}}_{\text{event}} \approx 10$ testing events under an 80/20 split, yielding $N_{\text{models}} = 20{,}000$ Random Forest models. The dataset is representative and additional subsets produced similar results. Hyperparameter specifications are available in the accompanying GitHub repository, and other configurations yielded comparable conclusions. Outcomes per pair fall into three categories: $\mathcal{D}_{\alpha \beta}$ better, equal, or worse than $\mathcal{D}_{\text{basic}}$. Of $N_{\text{pair}} = 20{,}000$ pairs we observed $12{,}739$ ties, $4{,}488$ better, and $2{,}773$ worse, giving the effective sample size $N' = 7{,}261$ for statistical testing.

\begin{table}[t]
\centering
\caption{Summary of statistical results comparing $\mathcal{D}_{\alpha \beta}$ to $\mathcal{D}_{\text{basic}}$.}
\label{tab:stats-table}
\begin{tabular}{lcc}
\toprule
Statistic & Value \\
\midrule
Effective sample size ($N'$) & 7,261 \\
$n_{\text{better}}$ & 4,488 \\
$n_{\text{worse}}$ & 2,773 \\
Mean ($\mu = N'/2$) & 3,630.5 \\
Std. dev. ($\sigma$) & 42.61 \\
$Z$-score & 20.13 \\
p-value & $< 10^{-89}$ \\
\bottomrule
\end{tabular}
\end{table}

Under $H_0$ with $p = 0.5$, a normal approximation gives mean $3{,}630.5$ and standard deviation $42.61$. We use the standard binomial test via the normal approximation with a continuity correction, customary for large $N'$ and $p=0.5$, which yields $Z \approx 20.13$ and $p < 10^{-89}$, as summarized in Table~\ref{tab:stats-table}. We observe a statistically significant improvement using $\mathcal{D}_{\alpha \beta}$ (Fig.~\ref{fig:alpha_beta_improvement}). The mean accuracy gain across $20{,}000$ models is modest yet consistent (Fig.~\ref{fig:subset_accuracy_comparison}). The large number of improved cases together with the extremely low p-value indicates a nontrivial effect. Adding $\alpha_i$ and $\beta_i$ improves artist discrimination and highlights the value of distributional features in dynamic pricing.

\begin{figure}[t]
  \centering
  \hspace*{\fill}
  \begin{subfigure}[c]{0.48\linewidth}
    \includegraphics[width=\linewidth]{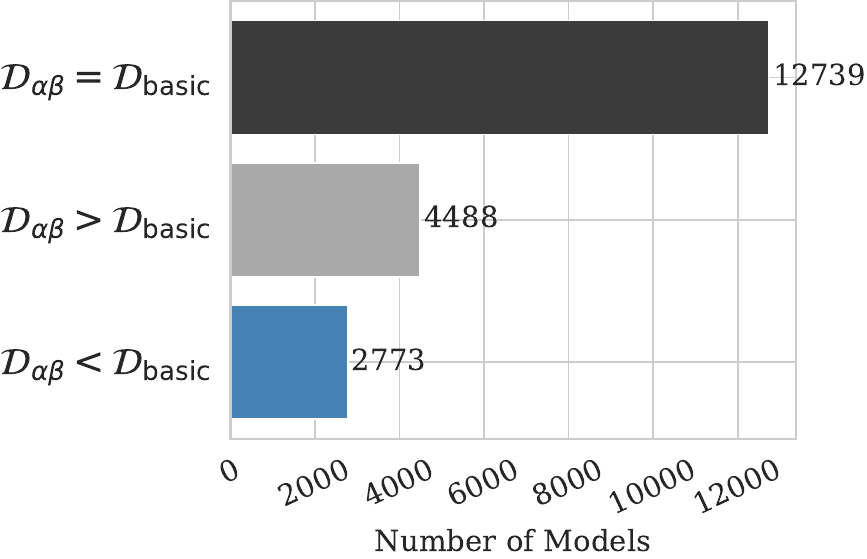}
    \caption{Random Forest model performance comparisons for $N_{\text{models}} = 20{,}000$,
    using the typical default feature selection size of $m = \text{round}(\sqrt{n}) = 2$. The bars show the number of cases in which models trained on $\mathcal{D}_{\alpha \beta}$ performed the same, better, or worse than models trained on $\mathcal{D}_{\text{basic}}$.}
    \label{fig:alpha_beta_improvement}
  \end{subfigure}
  \hspace*{\fill}
  \begin{subfigure}[c]{0.48\linewidth}
    \includegraphics[width=\linewidth]{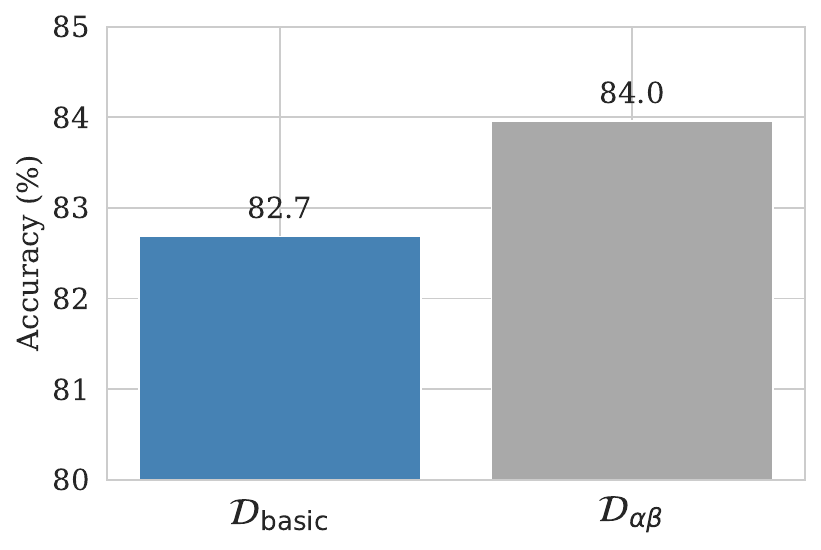}
    \caption{Percent Accuracy by Feature Subset for $N_{\text{models}} = 20{,}000$. Although the overall accuracy difference between $\mathcal{D}_{\alpha \beta}$ and $\mathcal{D}_{\text{basic}}$ appears relatively small, it is statistically significant given the large sample size and the substantial proportion of models showing improvement.}
    \label{fig:subset_accuracy_comparison}
  \end{subfigure}
  \hspace*{\fill}
  \caption{Random Forest performance comparison using $\mathcal{D}_{\text{basic}}$ vs. $\mathcal{D}_{\alpha \beta}$ features.}
  \label{fig:rf_comparison_ab_vs_basic}
\end{figure}

This supports the view that while \citep{dempster2024quant} exploit rich raw quantiles, limited summary statistics with principled estimation offer strong distributional characterization. Related quantile and moment findings \citep{zhang2023quantile,salimans2024moment} and classical beta inference \citep{krishnamoorthy2016handbook} align with these results, where carefully chosen statistics yield accurate parameters and improved classification.

\topicstart{Case study, Beyoncé vs. Ed Sheeran:}
To illustrate the value of incorporating estimated $\alpha_i$ and $\beta_i$, consider an Ed Sheeran concert on 6/29/2023 at the Boch Center Wang Theatre in Boston, MA. With only $\text{Min}_i$, $\text{Max}_i$, $\mu_i$, and $\tilde{\mu}_i$, the summary resembles a typical Beyoncé profile and the model misclassifies the event as Beyoncé (Fig.~\ref{fig:beyonce_sheeran_stats_comparison}). Adding the estimated parameters yields a scaled beta that exposes a sharper price drop characteristic of Ed Sheeran, which corrects the error (Fig.~\ref{fig:beyonce_sheeran_beta_comparison}). This demonstrates improved accuracy and robustness from integrating estimated distribution parameters into the Random Forest framework and it validates the inferred distribution.

\begin{figure}[t]
  \centering
  \hspace*{\fill}
  \begin{subfigure}[c]{0.48\linewidth}
    \includegraphics[width=\linewidth]{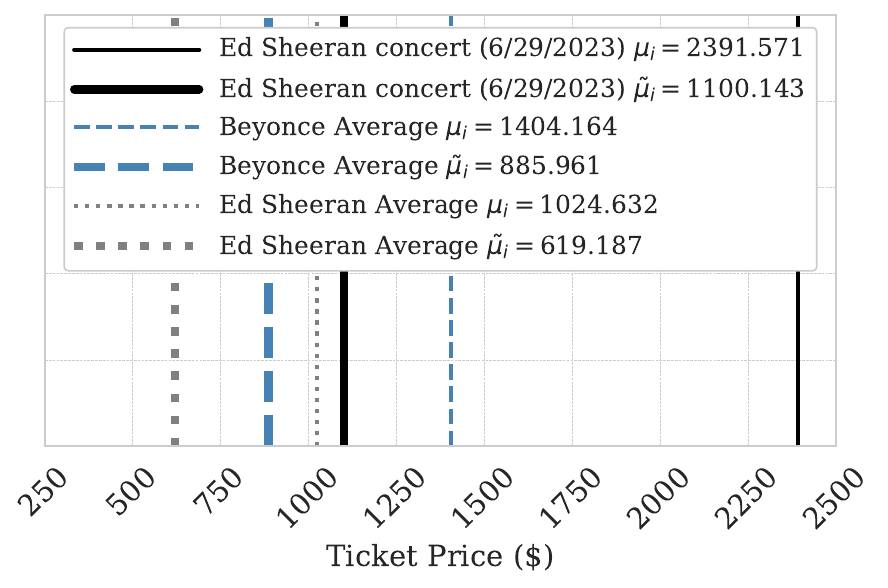}
    \caption{Ticket pricing summary for the Ed Sheeran concert on 6/29/2023 at Boch Center Wang Theatre, Boston, MA, using basic statistics ($\mu_i$, $\tilde{\mu}_i$, $\text{Min}_i$, $\text{Max}_i$). Without distribution parameters, the mean and median prices align closely with typical Beyoncé concert values, leading to misclassification.}
    \label{fig:beyonce_sheeran_stats_comparison}
  \end{subfigure}
  \hspace*{\fill}
  \begin{subfigure}[c]{0.48\linewidth}
    \includegraphics[width=\linewidth]{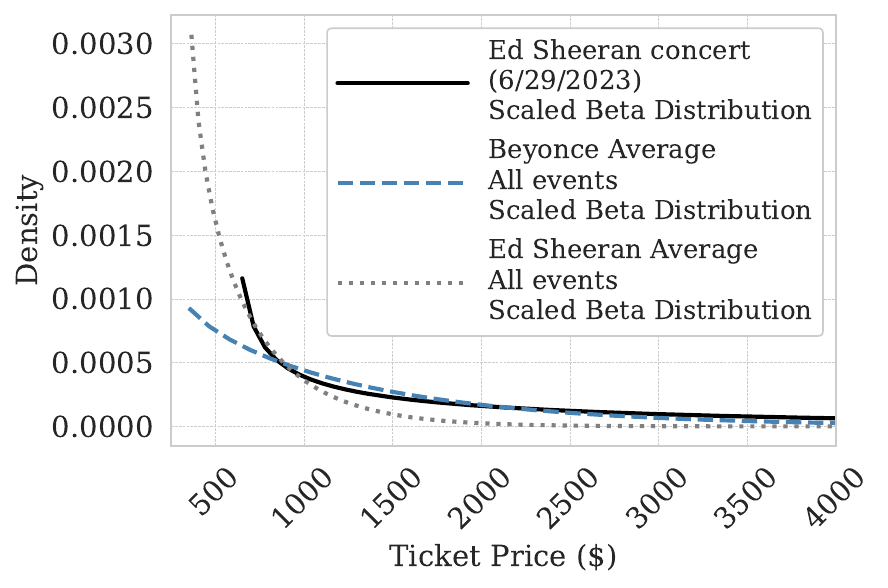}
    \caption{Comparison of scaled beta distributions after estimating $\alpha_i$ and $\beta_i$ parameters for the Ed Sheeran concert (6/29/2023). The estimated distribution shows a more pronounced price drop relative to the typical Beyoncé concert profile, accurately reflecting Ed Sheeran’s pricing pattern and correcting the previous misclassification.}
    \label{fig:beyonce_sheeran_beta_comparison}
  \end{subfigure}
  \hspace*{\fill}
  \caption{Statistical vs. distributional pricing representations for the Ed Sheeran concert on 6/29/2023 at Boch Center Wang Theatre.}
  \label{fig:beyonce_sheeran_comparison}
\end{figure}

\subsection{Synthetic ground-truth validation (scaled beta)}
\label{sec:synthetic_validation}

To verify the accuracy–fidelity link without relying on unknown SeatGeek densities, we construct a controlled experiment where the ground-truth law is a scaled beta on $[\text{Min}_i,\text{Max}_i]$ with parameters $\theta=(\alpha_i,\beta_i)$. We then mimic the estimation pipeline of Sect.~\ref{sec:estimation} to reconstruct $\hat\theta=(\alpha'_i,\beta'_i)$ from the limited statistics $(\mu_i,\tilde\mu_i,\text{Min}_i,\text{Max}_i)$.

\topicstart{Loss–divergence relations:}
Let $P_{\theta}$ and $P_{\hat\theta}$ denote the true and reconstructed densities on $[\text{Min}_i,\text{Max}_i]$. We measure total variation and Jensen-Shannon divergence
\[
\mathrm{TV}(P_{\hat\theta},P_{\theta})
=\tfrac12\!\int_{\text{Min}_i}^{\text{Max}_i}\!\!\big|p_{\hat\theta}(x)-p_{\theta}(x)\big|\,dx,
\qquad
\mathrm{JS}(P_{\hat\theta}\,\|\,P_{\theta}),
\]
and, using natural logarithms (JS in nats), apply the two-sided bounds
\begin{equation}
\tfrac12\,\mathrm{TV}^2(P_{\hat\theta},P_{\theta})
\;\le\;
\mathrm{JS}(P_{\hat\theta}\,\|\,P_{\theta})
\;\le\;
(\ln 2)\,\mathrm{TV}(P_{\hat\theta},P_{\theta}).
\end{equation}
For a labeled example with predicted probability $\hat p_i$ and logistic loss $\ell_i=-\big[y_i\log\hat p_i+(1-y_i)\log(1-\hat p_i)\big]$, the exact identity
\begin{equation}
\mathrm{TV}_i=\lvert y_i-\hat p_i\rvert \;=\; 1-e^{-\ell_i}
\end{equation}
implies $\ell=-\ln(1-\mathrm{TV})$ (small-error expansion: $\mathrm{TV}_i=\ell_i+O(\ell_i^2)$). Composing with the bounds yields $\mathrm{JS}_i=\tfrac12\,\ell_i^2+O(\ell_i^3)$ in the small-error regime, i.e., TV contracts linearly with loss, while JS contracts quadratically. These identities are model-agnostic; because our Random Forest evaluations use probabilistic cross-entropy (log loss) computed from the model’s predicted class probabilities, the loss–TV/JS relationships above apply verbatim to the Random Forest results in this work.

\topicstart{Chain of implications:}
The validation proceeds along a single flow from parameters to divergence: Lipschitz continuity of the feature map and density w.r.t.\ $(\alpha,\beta)$ controls how parameter error propagates (Sect.~\ref{sec:theoretical_bounds}); Theorem~\ref{thm:total_variation} (classification–TV control) links misclassification to distributional discrepancy; and JS sharpens TV quadratically at small error in Theorem~\ref{thm:jensen_shannon}. We summarize this pipeline as
\[
\underbrace{\sqrt{(\alpha'-\alpha)^2 + (\beta'-\beta)^2}}_{\text{Scaled beta parameter error}}
\;\xrightarrow{\;L\;}\;
\underbrace{\varepsilon = \mathbb{P}(E)}_{\text{classification error}}
\;\xrightarrow{\text{Theorem~\ref{thm:total_variation}}}\;
\underbrace{\mathrm{TV}(P_\theta,P_{\hat\theta})}_{\text{distributional distance}}
\;\xrightarrow{\text{Theorem~\ref{thm:jensen_shannon}}}\;
\underbrace{\mathrm{JS}(P_\theta,P_{\hat\theta})}_{\text{divergence fidelity}}.
\]

Figure~\ref{fig:beta_overlay_true_vs_perturbed} anchors the construction by overlaying the true density with a divergence-sorted family of reconstructed curves, making visible how departures from $(\alpha_i,\beta_i)$ alter the shape on the observed range. Figure~\ref{fig:js_vs_tvd_envelope} plots empirical $\mathrm{JS}$ against $\mathrm{TV}$ with a secondary loss axis $\ell=-\ln(1-\mathrm{TV})$, and the curve lies between the quadratic lower and linear upper bounds, confirming the globally linear and small-error quadratic regimes predicted by the theory. Together, these figures provide a compact, direct verification that out-of-sample accuracy is a reliable surrogate for distributional fidelity when only $(\mu_i,\tilde\mu_i,\text{Min}_i,\text{Max}_i)$ are available.

\begin{figure}[t]
  \centering
  \hspace*{\fill}
  \begin{subfigure}[c]{0.48\linewidth}
    \includegraphics[width=\linewidth]{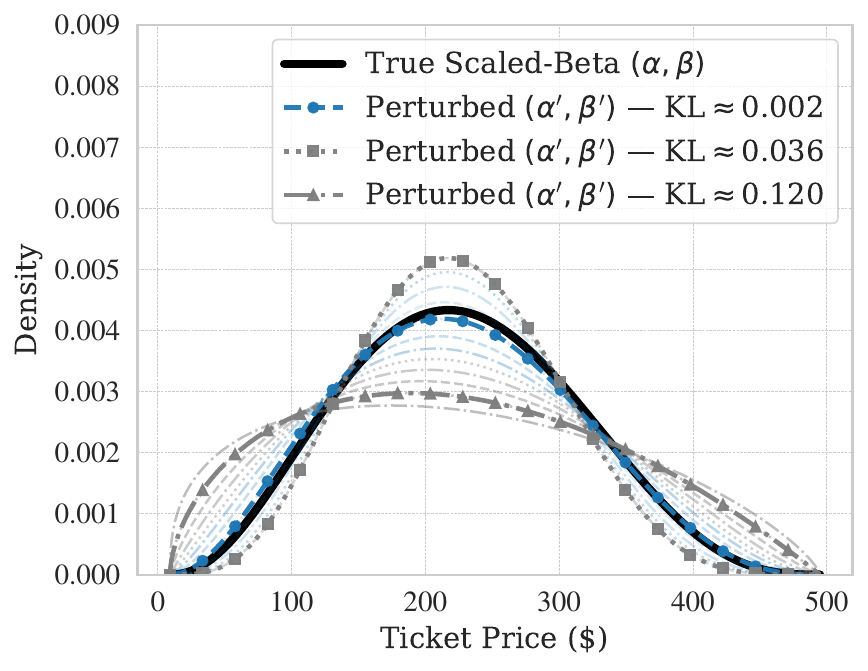}
    \caption{True vs.\ reconstructed scaled beta densities (divergence-sorted). 
    Ground-truth $P_{\theta}$ (black) on $[\text{Min}_i,\text{Max}_i]$ with $(\alpha_i,\beta_i)$ 
    and reconstructed $P_{\hat\theta}$ curves (blue$\rightarrow$grey) obtained from 
    $(\mu_i,\tilde\mu_i,\text{Min}_i,\text{Max}_i)$. 
    Shading reflects increasing divergence to truth and visually motivates the applied distance measures.}
    \label{fig:beta_overlay_true_vs_perturbed}
  \end{subfigure}
  \hspace*{\fill}
  \begin{subfigure}[c]{0.48\linewidth}
    \includegraphics[width=\linewidth]{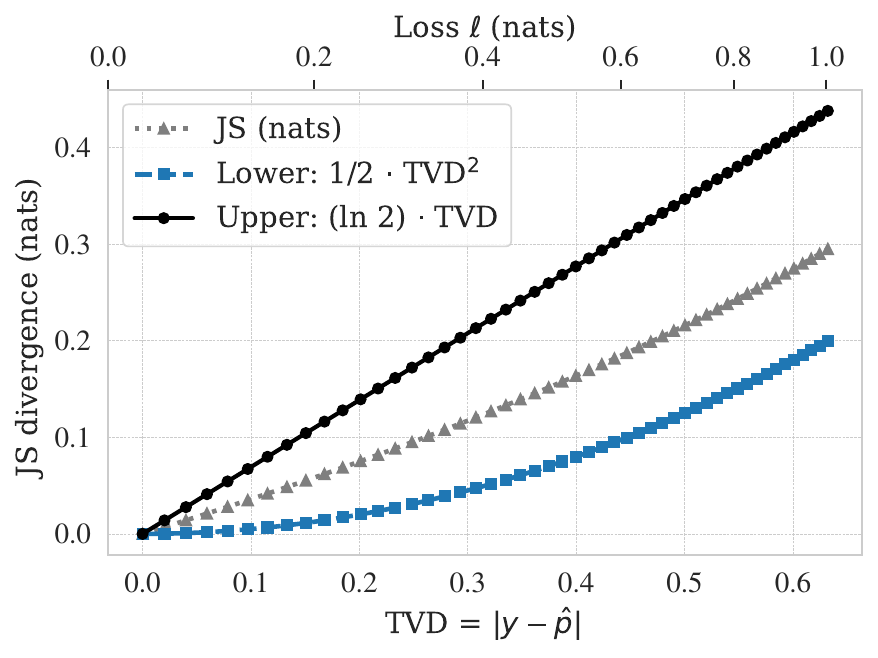}
    \caption{JS vs.\ TV with bounds and secondary loss axis. 
    Empirical $\mathrm{JS}(P_{\hat\theta}\,\|\,P_{\theta})$ lies within the two-sided bounds 
    $\tfrac{1}{2}\mathrm{TV}^2 \le \mathrm{JS} \le (\ln 2)\mathrm{TV}$ (nats). 
    The top axis maps TV to logistic loss via $\ell=-\ln(1-\mathrm{TV})$, 
    tying divergence trends directly to classification confidence.}
    \label{fig:js_vs_tvd_envelope}
  \end{subfigure}
  \hspace*{\fill}
  \caption{Distributional divergence analysis.
  (a) shows reconstructed vs.\ true scaled beta densities sorted by divergence; 
  (b) plots Jensen-Shannon divergence against total variation with theoretical bounds 
  and a mapped loss axis. Together these illustrate how reconstruction fidelity relates 
  to classification-relevant divergence scales.}
  \label{fig:divergence_analysis_pair}
  \vskip -0.15in
\end{figure}

\section{Implicit regularization via zero-variance features}
\label{sec:regularization}

Regularization improves generalization by limiting overfitting to noisy or irrelevant features. In Random Forests, bagging \citep{breiman1996bagging} and random subspace sampling \citep{ho1998rsm} yield implicit regularization and robustness. Bootstrapped trees reduce variance, and random feature choice decorrelates trees and controls complexity. \citep{breiman2001random} decomposed generalization error into tree strength and inter-tree correlation, providing a theoretical basis.

\begin{equation}
\mathbb{E}[(f(x) - \mathbb{E}[y \mid x])^2] 
= \mathrm{Var}(f(x)) + \mathrm{Bias}^2(f(x)) + \sigma^2,
\end{equation}

where $\mathbb{E}[y \mid x]$ is the true conditional class probability, $\mathrm{Var}(f(x))$ the model variance, $\mathrm{Bias}^2(f(x))$ the squared bias, and $\sigma^2$ irreducible noise. ``Extremely Randomized Trees'' increase decorrelation and lower variance \citep{geurts2006extremely}, while stability results \citep{bousquet2002stability} support varied ensembles. We study another implicit mechanism by adding zero-variance (constant) features to reshape feature selection and decision boundaries. In ticket pricing, this serves as a lever to tune trust in recovered distributions, since adjustments to the Random Forest ensemble directly modulate the fidelity of the inferred parameters.

\subsection{Analysis of probabilistic feature selection}

\topicstart{Notation:} For notational clarity in what follows, 
we let $ m = \texttt{max\_features} $ denote the number of features 
randomly selected at each split in the Random Forest, following the 
scikit-learn \citep{scikit-learn} implementation.

In a standard Random Forest construction with fixed-size feature selection, exactly $m$
out of the $n$ total features are chosen at each split node, uniformly over all
$\binom{n}{m}$ subsets. Hence, the probability that a particular feature $X_j$ is included
in the candidate set at any node is
\begin{equation}
P(\text{Include } X_j) 
\;=\; 
\frac{m}{n}.
\end{equation}
Across $B$ trees, each containing an average of $L$ split nodes, the expected total number
of times $X_j$ appears in candidate sets is then
\begin{equation}
\mathbb{E}[\mathrm{Count}_{\text{in-candidate}}(X_j)]
\;=\;
B \,\cdot\, L \,\cdot\, \frac{m}{n}.
\end{equation}

\topicstart{Feature selection via gini impurity reduction:} For binary
classification, the Gini impurity is
\begin{equation}
G
\;=\;
2\,p\,(1-p),
\end{equation}
where $p$ is the proportion of one class. Splitting on $X_j$ changes this impurity,
reducing it by $\Delta G(X_j)$. We define the rank or score of feature
$X_j$ as
\begin{equation}
r(X_j)
\;=\;
\Delta G(X_j).
\end{equation}
A higher $r(X_j)$ means a larger impurity reduction and thus a higher rank among the
available features at a node. Moderately predictive parameters, such as $\alpha_i$ and
$\beta_i$ in the artist classification use-case, can still achieve some positive $r(X_j)$,
even if not as large as top-ranked features.

\topicstart{Competitive advantage of highly ranked features:} Although each feature
$X_j$ has a nominal $\tfrac{m}{n}$ chance of appearing in the size-$m$ candidate set
at a node, the final split is awarded to whichever feature yields the greatest score. If we
assume a proportional ``weighted by $r(X_j)$'' selection among the $m$ chosen, then for a
subset
\begin{equation}
S
\;\subseteq\;
\{1,\dots,n\},
\quad
|S|
\;=\;
m,
\end{equation}
we have
\begin{equation}
P(S)
\;=\;
\frac{1}{
  \binom{n}{m}
}.
\end{equation}
Conditioned on $S$, the probability that $X_j$ wins the split is
\begin{equation}
P(X_j \mid S)
\;=\;
\frac{
  r(X_j)
}{
  \sum_{k \in S} r(X_k)
}.
\end{equation}
Hence, the unconditional probability of $X_j$ being chosen for a split is
\begin{equation}
P(X_j)
\;=\;
\sum_{\substack{S:\,j\in S}}
P(S)\,\cdot\, P(X_j \mid S),
\end{equation}
which expands to
\begin{equation}
P(X_j)
\;=\;
\frac{1}{
  \binom{n}{m}
}
\;\sum_{\substack{S:\,j\in S}}
\frac{
  r(X_j)
}{
  \sum_{k \in S} r(X_k)
}.
\end{equation}

\topicstart{Closed-form approximation:} When $n \gg m$, or simply for conceptual
ease, one can approximate $\sum_{k \in S} r(X_k)$ by its expectation,
\begin{equation}
r(X_j)
\;+\;
(m-1)\,\mathbb{E}[\,r(X_k)\bigr],
\end{equation}
yielding
\begin{equation}
P(X_j)
\;\approx\;
\end{equation}
\begin{equation}
\frac{1}{
  \binom{n}{m}
}
\;\sum_{\substack{S:\,j\in S}}
\frac{
  r(X_j)
}{
  r(X_j)
  + (m-1)\,\mathbb{E}[\,r(X_k)\bigr].
}
\end{equation}
Since there are $\binom{n-1}{m-1}$ subsets that include $X_j$, and
$\tfrac{\binom{n-1}{m-1}}{\binom{n}{m}} = \tfrac{m}{n},$ we obtain
\begin{equation}
P(X_j)
\;\approx\;
\frac{m}{n}
\,\cdot\,
\frac{
  r(X_j)
}{
  r(X_j)
  + (m-1)\,\mathbb{E}[\,r(X_k)\bigr].
}
\end{equation}
Thus, even though every feature has the same nominal $\tfrac{m}{n}$ chance of
entering the candidate set, those with consistently higher $r(X_j)$ can dominate,
overshadowing less ranked predictors.

\topicstart{Probabilistic effects of zero-variance variables:} Earlier (Section \ref{sec:preliminaries}), the datasets $\mathcal{D}_{\alpha\beta}$ and $\mathcal{D}_{\delta}$ were extended to $\mathcal{D}_{\alpha\beta}^{(\text{reg})}$ and $\mathcal{D}_{\delta}^{(\text{reg})}$ by adding zero-variance (constant-value) features. The count of such features is $n_{\mathrm{ZV}}$. They have near-zero Gini scores because they yield no impurity reduction. Although seemingly unhelpful, these variables shift the Random Forest selection dynamics by adding low-score competition that tempers dominance of top-ranked features. Let $n$ denote the non-constant features (e.g., $\mu_i, \tilde{\mu}_i, \text{Max}_i, \text{Min}_i, \alpha_i, \beta_i$) and $n_{\mathrm{ZV}}$ the zero-variance ones, so the feature set has $n + n_{\mathrm{ZV}}$ elements. If each zero-variance feature has rank score $r_{\mathrm{ZV}}\approx 0$, it dilutes the score sum in the denominator and boosts the selection probability of mid-ranked features compared to the case with no zero-variance features.

\begin{theorem}[Zero-Variance Dilution Effect]
\label{thm:zero_variance_dilution}
Suppose $n_{\mathrm{ZV}}$ zero-variance features with $r_{\mathrm{ZV}}\!\approx\!0$ are
added, enlarging the feature set from $n$ to $n_{\mathrm{eff}} = n+n_{\mathrm{ZV}}$.
Let
\begin{equation}
\overline{r}_{\mathrm{eff}}
=\frac{\textstyle \sum_{j=1}^{n} r(X_j) + n_{\mathrm{ZV}}\,r_{\mathrm{ZV}}}
       {n_{\mathrm{eff}}}
\;\approx\;
\frac{n\,\overline{r}}{n+n_{\mathrm{ZV}}}
\end{equation}
\begin{equation}
\text{so that}\quad\overline{r}_{\mathrm{eff}} < \overline{r}.
\end{equation}
For any two informative features $X_h,X_\ell$ with scores
$a=r(X_h) > b=r(X_\ell) > r_{\mathrm{ZV}}$, the closed-form odds ratio
between their selection probabilities satisfies
\begin{equation}
\frac{P_h^{(\mathrm{eff})}(m)}{P_\ell^{(\mathrm{eff})}(m)}
\;<\;
\frac{P_h(m)}{P_\ell(m)},
\end{equation}
where
\begin{equation}
P_j(m)           = \frac{m}{n}\,
                    \frac{r(X_j)}{\,r(X_j)+(m-1)\,\overline{r}},
\qquad
\end{equation}
\begin{equation}
P_j^{(\mathrm{eff})}(m)
                  = \frac{m}{n_{\mathrm{eff}}}\,
                    \frac{r(X_j)}{\,r(X_j)+(m-1)\,\overline{r}_{\mathrm{eff}}}.
\end{equation}

Thus adding zero-variance features compresses the relative dominance of
higher-scoring over lower-scoring variables, giving mid-ranked features more
splitting opportunities.
\end{theorem}

\begin{proof}
Write $K=(m-1)\,\overline{r}$ and $K_t=(m-1)\,\overline{r}_{\mathrm{eff}}$
with $K_t<K$.
The prefactors $\tfrac{m}{n}$ and $\tfrac{m}{n_{\mathrm{eff}}}$
cancel in the ratio, giving
\begin{equation}
\frac{P_h^{(\mathrm{eff})}(m)}{P_\ell^{(\mathrm{eff})}(m)}
      =\frac{a}{b}\,
       \frac{\,b+K_t\,}{\,a+K_t\,},
\quad
\frac{P_h(m)}{P_\ell(m)}
      =\frac{a}{b}\,
       \frac{\,b+K\,}{\,a+K\,}.
\end{equation}
Define
\begin{equation}
R(K)=\dfrac{a}{b}\dfrac{\,b+K\,}{\,a+K\,}.
\end{equation}
A direct derivative gives
\begin{equation}
\dfrac{dR}{dK}= \dfrac{a(a-b)}{b\,(a+K)^2}>0
\end{equation}
because $a>b>0$; hence $R(K)$ is strictly increasing in $K$.
Since $K_t<K$, we have 
\begin{equation}
R(K_t)<R(K),
\end{equation}
establishing the claimed inequality.
\end{proof}

\citep{geurts2006extremely} demonstrated that increasing the randomization of split selection in Extremely  
Randomized Trees leads to deeper decision trees by weakening the dependence of split choices on the  
target variable. This increased depth arises because random splits reduce the impurity reduction at each  
node, thus requiring additional splits to achieve sufficient purity. Formally, this can be expressed as:
$
\mathbb{E}[d_{\text{random}}] > \mathbb{E}[d_{\text{optimal}}]
$, 
where $\mathbb{E}[d_{\text{random}}]$ and $\mathbb{E}[d_{\text{optimal}}]$ represent the expected tree depths for randomized  
and optimal splits, respectively. Our approach and experiments reveal comparable  
effects, with zero-variance features increasing tree depth and encouraging more variation among splits.

\begin{corollary}[Increased Expected Tree Depth]
\label{corollary:increased_tree_depth}
Consider a Random Forest whose effective feature set is
$
n_{\mathrm{eff}} = n + n_{\mathrm{ZV}},
$
with $n$ informative features ($r(X_j)>0$) and $n_{\mathrm{ZV}}$ zero-variance
features ($r_{\mathrm{ZV}}\!\approx\!0$).
Let $d$ denote the depth of a decision tree grown under a fixed impurity-based
stopping rule.  Then, holding all other training hyperparameters constant,
\begin{equation}
\mathbb{E}\!\bigl[d(n_{\mathrm{eff}})\bigr] \;>\;
\mathbb{E}\!\bigl[d(n)\bigr].
\end{equation}
\end{corollary}

\begin{proof}
Theorem~\ref{thm:zero_variance_dilution} shows that adding zero-variance features
compresses the odds of high- versus mid-ranked variables:
\begin{equation}
\frac{P_h^{(\mathrm{eff})}(m)}{P_\ell^{(\mathrm{eff})}(m)}
\;<\;
\frac{P_h(m)}{P_\ell(m)} .
\end{equation}
Consequently, top-scoring features win fewer splits relative to before, and
more mid-ranked features are selected.  
Because those mid-ranked features achieve smaller impurity reductions
($r_\ell<b<r_h$), the expected impurity drop per internal node is lower.
A lower per-split reduction means the chosen impurity threshold is reached
later in the recursive partitioning process, so additional levels are needed
before termination.  Hence the expected depth increases:
$
\mathbb{E}[d(n_{\mathrm{eff}})]>\mathbb{E}[d(n)].
$
\end{proof}

Increased randomness in split selection  
explicitly reduces the correlation among trees \citep{geurts2006extremely}, expressed mathematically as:
\begin{equation}
\sigma^2 = \rho\,\frac{\mathrm{Var}(h(x))}{B},
\end{equation}
where lower correlation $\rho$ directly reduces ensemble variance $\sigma^2$.
Our theoretical analysis and empirical results confirm this assertion.

\begin{corollary}[Reduced Ensemble Correlation]
\label{corollary:reduced_ensemble_correlation}
Let $\rho$ denote the pair-wise correlation between base learners in a
Random Forest.  Adding $n_{\mathrm{ZV}}$ zero-variance features lowers that
correlation and hence the variance term
$\rho\,\mathrm{Var}(h)/B$
in Breiman’s bias–variance decomposition.
\end{corollary}

\begin{proof}
Without zero-variance features, the highest-ranked variables win a large
fraction of candidate splits; many trees therefore grow similar decision
paths, inflating $\rho$.  
Theorem~\ref{thm:zero_variance_dilution} shows that after augmentation the
odds ratio
$
P_h^{(\mathrm{eff})}(m)/P_\ell^{(\mathrm{eff})}(m)
$
shrinks for every pair of scores $a>b>0$.
Consequently, top-ranked variables win fewer splits relative to mid-ranked
ones, and different features now have a greater chance of initiating branches.
This increased heterogeneity of split choices makes the predictions of
individual trees less correlated, so $\rho$ decreases; the factor
$\rho\,\mathrm{Var}(h)/B$ is therefore reduced.
\end{proof}

\subsection{Expanding the regularization search space}
\label{sec:continuous_search_space}

Prior work by \citep{mentch2019quantifying} shows that tuning the $m$ ($\mathtt{max\_features}$) hyperparameter, the number of features considered at each split, regularizes Random Forests by altering the chance that a feature is selected. Because $m$ is an integer, their scheme moves in discrete steps and yields a finite set of selection probabilities. Our approach adds constant-value features, which changes the total feature count and yields a near continuum of expected feature-selection probabilities. The theorem below formalizes that introducing $n_{\mathrm{ZV}}$ constant features can approximate any target probability in a broad interval by randomized interpolation between adjacent choices of $n_{\mathrm{ZV}}$, thereby ``filling in the gaps'' left by discrete $m$ adjustments.

\begin{theorem}[Continuous Approximation via Zero-Variance Dilution]
\label{thm:continuous_approx}
Let $n$ denote the number of truly informative features, and fix an integer 
$m$ such that $1 \le m \le n$. Let $n_{\mathrm{ZV}}$ be the number of constant (zero-variance) features added.
In the absence of constant features, the effective probability of selecting an informative feature at a split is
\begin{equation}
  \gamma \;=\; \frac{m}{n}.
\end{equation}
If we add $n_{\mathrm{ZV}}\ge 0$ constant (zero-variance) features, then the total number of features 
is $n + n_{\mathrm{ZV}}$, and the effective selection probability of an informative feature becomes
\begin{equation}
  \gamma' \;=\; \frac{m}{n + n_{\mathrm{ZV}}}.
\end{equation}
The set
\begin{equation}
  S_{L} 
  \;=\;
  \Bigl\{ \,\tfrac{m}{\,n + n_{\mathrm{ZV}}} : n_{\mathrm{ZV}} \in \mathbb{N}_0 \Bigr\}
\end{equation}
is a countable, monotone grid spanning $\bigl(0,\tfrac{m}{n}\bigr]$ with $0$ as its only accumulation point.
Moreover, for any desired probability
\begin{equation}
  0 \;<\; \gamma^* \;\le\; \frac{m}{n},
\end{equation}
and any $\epsilon>0$, there exist adjacent integers $k,k\!+\!1$ and a mixing weight $p\in[0,1]$ such that the randomized scheme
that uses $n_{\mathrm{ZV}}=k$ with probability $p$ and $n_{\mathrm{ZV}}=k+1$ with probability $1-p$ achieves
\begin{equation}
  \Bigl|\,\mathbb{E}[\gamma'] \;-\; \gamma^* \Bigr| \;<\; \epsilon,
\end{equation}
where $\mathbb{E}[\gamma']= p\,\tfrac{m}{n+k} + (1-p)\,\tfrac{m}{n+k+1}$. Hence, by randomized interpolation between adjacent grid points,
the expected selection probability can be tuned arbitrarily finely over $\bigl(0,\tfrac{m}{n}\bigr]$.
\end{theorem}

\begin{proof}
\topicstart{Discrete set without constant features:}
Following \citep{mentch2019quantifying}, let $n$ be the number of 
informative features and let $m$ be the chosen subset size at each split.
The probability that any one informative feature appears in a candidate set is then
$\gamma = \tfrac{m}{n}$. Because $m$ must be an integer with $1 \le m \le n$, 
the set of possible probabilities (as $m$ varies) is 
\begin{equation}
  S_{MZ} 
  \;=\;
  \Bigl\{
     \tfrac{1}{n}, \tfrac{2}{n}, \dots, \tfrac{n}{n}
  \Bigr\}.
\end{equation}
This set is finite and discrete.

\medskip
\topicstart{Zero-variance dilution and the grid $S_L$:}
Fix $m$. Instead of varying $m$ itself, we add $n_{\mathrm{ZV}}$ constant 
(zero-variance) features to the existing $n$ informative ones, so the total 
feature count is $n + n_{\mathrm{ZV}}$. As a result, the effective probability of picking an 
informative feature becomes
\[
  \gamma' 
  \;=\; 
  \frac{m}{n + n_{\mathrm{ZV}}}.
\]
Hence, each nonnegative integer $n_{\mathrm{ZV}}$ in $\{0,1,2,\dots\}$ produces one element 
of the set
\[
  S_{L} 
  \;=\;
  \Bigl\{ \tfrac{m}{n + n_{\mathrm{ZV}}} : n_{\mathrm{ZV}} \in \mathbb{N}_0 \Bigr\}.
\]
Because $n_{\mathrm{ZV}}$ can grow arbitrarily large, the values of $\gamma'$ can get 
arbitrarily close to $0$. Also, when $n_{\mathrm{ZV}}=0$, $\gamma'=\frac{m}{n}=\gamma$. 
Thus, $S_{L}$ spans probabilities in $(0, \frac{m}{n}]$ and forms a countable, monotone grid with $0$ as its only accumulation point.

\medskip
\topicstart{Randomized interpolation to approximate any target $\gamma^*$:}
Take any target probability $\gamma^*$ satisfying $0 < \gamma^* \le \frac{m}{n}$ and any $\epsilon>0$.
Choose $k = \bigl\lfloor \tfrac{m}{\gamma^*} - n \bigr\rfloor \vee 0$. Then $\tfrac{m}{n+k} \ge \gamma^* \ge \tfrac{m}{n+k+1}$, so $\gamma^*$ lies between the adjacent grid points.
Define
\begin{equation}
  p \;=\; \frac{\gamma^* - \tfrac{m}{n+k+1}}{\tfrac{m}{n+k} - \tfrac{m}{n+k+1}} \;\in\; [0,1].
\end{equation}
If we select $n_{\mathrm{ZV}}=k$ with probability $p$ and $n_{\mathrm{ZV}}=k+1$ with probability $1-p$, then
\begin{equation}
  \mathbb{E}[\gamma'] \;=\; p\,\tfrac{m}{n+k} + (1-p)\,\tfrac{m}{n+k+1} \;=\; \gamma^*,
\end{equation}
and thus $\bigl|\mathbb{E}[\gamma']-\gamma^*\bigr|=0<\epsilon$. When $\gamma^*$ is very small, increasing $k$ makes the adjacent grid spacing arbitrarily fine; the same interpolation then achieves any prescribed $\epsilon>0$. This establishes the claimed approximation.
\end{proof}

\begin{corollary}[Continuous Accuracy Expansion via Selection Probability] 
\label{corollary:continuous_accuracy}
  Let 
  \[\gamma' = \frac{m}{n+n_{\mathrm{ZV}}}\] 
  be the effective probability of selecting an
  informative feature when the original $n$ features are augmented with $n_{\mathrm{ZV}}$
  constant (zero-variance) features. Assume that the mapping from $\gamma'$ to
  the classifier's accuracy $\nu$ is continuous, and let
  $(\nu_{\min}, \nu_{\max}]$ denote the interval of achievable accuracy
  values under the original discrete scheme. Then, for any target accuracy
  $\nu^*$ satisfying 
  \begin{equation}
  \nu_{\min} < \nu^* \le \nu_{\max},
  \end{equation} 
  and for any $\epsilon > 0$, there exist adjacent integers $k,k\!+\!1$ and a mixing weight $p\in[0,1]$
  such that the randomized scheme using $n_{\mathrm{ZV}}\in\{k,k\!+\!1\}$ with probabilities $(p,1-p)$ yields
  an expected accuracy $\mathbb{E}[\nu']$ (as a function of $\mathbb{E}[\gamma']$) satisfying
  \begin{equation}
  \bigl|\mathbb{E}[\nu'] - \nu^*\bigr| < \epsilon.
  \end{equation} 
  In other words, the set of expected achievable
  accuracies is dense in $(\nu_{\min}, \nu_{\max}]$, providing near-continuous
  control over the model's performance by fine-tuning the effective selection
  probability.
\end{corollary}

Adjusting $m$ influences a Random Forest's effective complexity \citep{mentch2019quantifying}, yet the integer nature of $m$ limits how finely one can tune selection probabilities. Expanding the feature set from $n$ to $n+n_{\mathrm{ZV}}$ with zero-variance features yields $\gamma'$ values for any integer $n_{\mathrm{ZV}} \ge 0$. These values form a monotone grid on $\bigl(0,\,\tfrac{m}{n}\bigr]$ and can be randomly interpolated to match any target in expectation, giving near-continuous control over the effective regularization level. This can mimic or surpass the effect of a small $m$ while refining probabilities beyond what integer steps allow. Related continuous approximations from discrete spaces appear in hyperparameter tuning, where refined grids approximate continuous optimization \citep{cironis2022automatic}, broadening the applicability of discrete-choice methods in machine learning.

\subsection{Connections to penalty methods, functional data analysis, and related areas}
Our approach connects fundamentally to classic regularization methods, such as ridge regression
\citep{hoerl1970ridge, tikhonov1943stability}, where explicit quadratic penalties emerge naturally
from Gaussian priors:
\begin{equation}
x_{MAP} = \arg \min_x \{\|Ax - b\|^2 + \lambda \|x\|^2\}.
\label{eq:ridge_regression}
\end{equation}

Extending beyond foundational work, the feature probability reweighting structurally parallels
recent advancements in regularization across various domains. For instance, in functional data
analysis (FDA), recent methods such as roughness penalization in free-knot spline estimation
\citep{magistris2024} redistribute information to avoid over-concentration on specific knots,
maintaining balanced representations. Similarly, our implicit regularization dynamically adjusts
feature selection probabilities, preventing dominance by specific features:
\begin{equation}
p'_i = \frac{p_i}{1 + \lambda \sum_{j} p_j},
\label{eq:feature_reweighting}
\end{equation}
This formulation resembles penalized optimization used in FDA:
\begin{equation}
C = \arg \min_C \|Y - \Phi^T C\|^2 + \lambda C^T R C,
\label{eq:FDA_penalty}
\end{equation}
which explicitly penalizes abrupt variations to enforce smoothness.

Furthermore, our implicit feature-selection regularization is related to another penalty-based approach, the inverse-problem
hyperparameter optimization framework introduced by \citep{dunbar2025hyperparameter},
whose formulation includes a log-determinant regularization:
\begin{equation}
L_M^{(EKI)}(u) = \|\Gamma(u)^{-1/2}(z - G(u))\|^2 - 2 \log P(u).
\label{eq:dunbar_eki}
\end{equation}

The interplay between implicit and explicit regularization frameworks presents an intriguing intersection of 
theoretical and applied perspectives.

\subsection{Analogous effects in ticket pricing and digit classification}

The same phenomena appear in both (i) the ticket pricing dataset with distribution-summary features and (ii) handwritten digit classification. In these settings, highly ranked features (e.g., strong distributional predictors or informative pixel locations) dominate splits and suppress weaker but useful signals. Introducing zero-variance variables reduces this dominance, increasing the selection frequency of subtle features and enriching the model. These variables function as implicit regularization similar to setting $m = 1$: they lower the effective weight of very high-ranked features, raise the probability of choosing secondary ones, and increase splitting variety across the ensemble, which can improve performance.

\subsection{Experimental results}
\label{sec:regularization_experimental_results}

This section applies the same pairwise Random Forest classification methodology from the $\mathcal{D}_{\alpha \beta}$ experiments to test zero-variance features as an implicit regularizer in two domains: a new concert ticket pricing dataset and the UCI handwritten digits dataset \citep{uci_digits}. We compare $\mathcal{D}_{\alpha \beta}$ to $\mathcal{D}_{\alpha \beta}^{\text{(reg)}}$ on 5{,}000 artist-pair models spanning 954 artists, each trained and tested on an 80/20 split with $\approx 40$ training events and $\approx 10$ test events per model. We also examine $\mathcal{D}_{\delta}$ versus $\mathcal{D}_{\delta}^{(\text{reg})}$ on 90 digit-pair classifiers across 10 digits with a similar 80/20 split providing $\approx 287$ training and $\approx 72$ test comparisons per pair. Regularized models use $n_{\mathrm{ZV}}=20$. For $\mathcal{D}_{\delta}$ we select a random subset of $n=6$ representative features for consistency with $\mathcal{D}_{\alpha \beta}$. These complementary experiments show the consistently beneficial effect of implicit regularization via zero-variance features in both real-world secondary ticket pricing and a classic benchmark dataset.

\topicstart{Accuracy and the selection size, $m$:} 
Figs.~\ref{fig:combined_max_features1} and \ref{fig:combined_max_features2} show accuracy trends for the concert pricing and handwritten digit datasets as $m$ varies. Iterating $m$ in a standard Random Forest alters the chance of selecting informative features, yet our experiments find peak accuracy only with implicit regularization. Theorem~\ref{thm:continuous_approx} shows that adding constant (zero-variance) features expands tuning from discrete $m$ steps to a near continuum of selection probabilities, and Corollary~\ref{corollary:continuous_accuracy} confirms that this continuous space enables fine-grained performance adjustments. Selection probabilities reach extremes at the boundaries of $m$ (e.g., $m=1$ yields uniform selection, while $m=n$ favors highly ranked features), and our empirical results show that continuous tuning via implicit regularization adds flexibility. By incorporating constant features, the effective average rank of candidates is diluted, balancing dominant and less prominent predictors and achieving accuracies difficult to reach by iterating $m$ alone.

\begin{figure}[t]
  \centering
  \hspace*{\fill}
  \begin{subfigure}[c]{0.48\linewidth}
    \includegraphics[width=\linewidth]{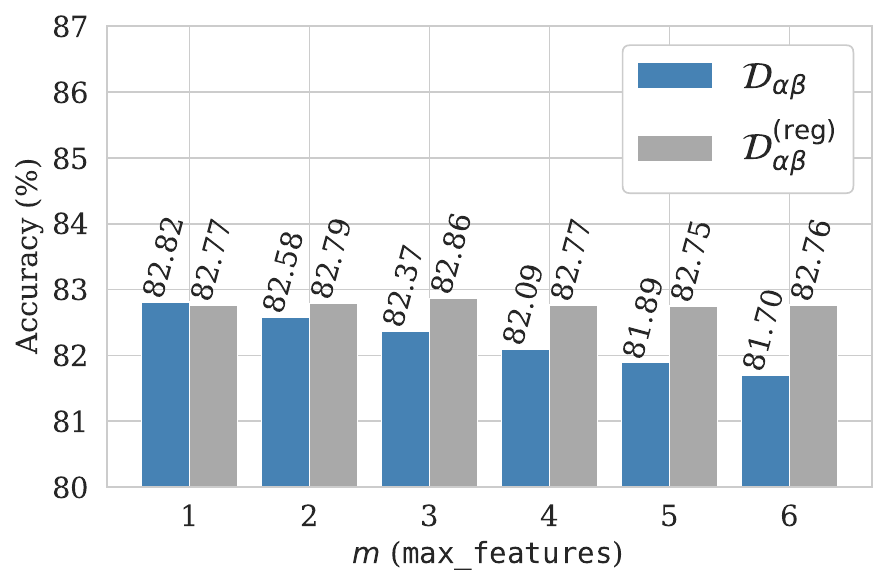}
    \caption{Percent accuracy across feature selection size ($m$) iterations for artist event training datasets 
    $\mathcal{D}_{\alpha \beta}$ and $\mathcal{D}_{\alpha \beta}^{\text{(reg)}}$. The impact of constant features 
    is evident, with $\mathcal{D}_{\alpha \beta}^{\text{(reg)}}$ showing declining accuracy as $m$ increases. The 
    highest accuracy is achieved using zero-variance feature regularization, exceeding what standard hyperparameter tuning alone can reach.}
    \label{fig:combined_max_features1}
  \end{subfigure}
  \hspace*{\fill}
  \begin{subfigure}[c]{0.48\linewidth}
    \includegraphics[width=\linewidth]{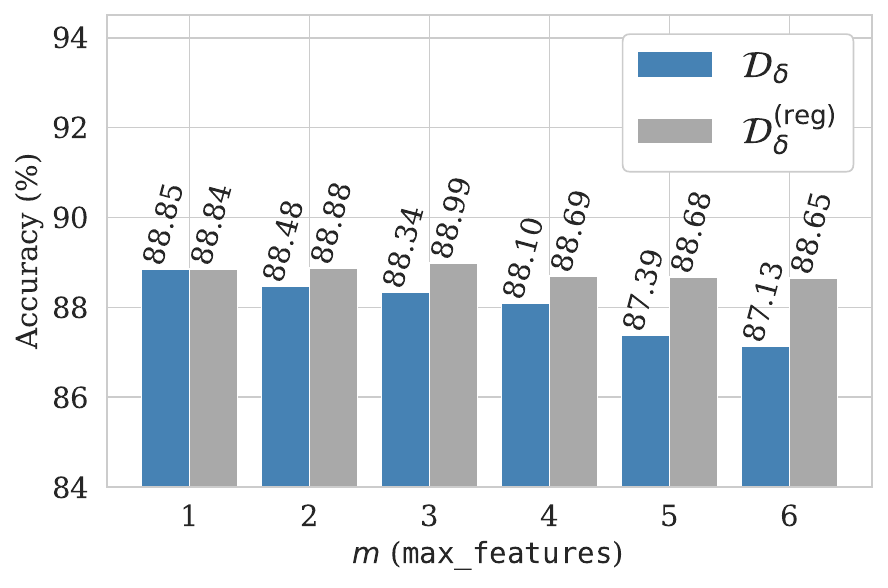}
    \caption{Percent accuracy across feature selection size ($m$) iterations for handwritten digit training datasets 
    $\mathcal{D}_{\delta}$ and $\mathcal{D}_{\delta}^{(\text{reg})}$. The impact of zero-variance features is 
    again evident, with $\mathcal{D}_{\delta}^{(\text{reg})}$ showing declining accuracy as $m$ increases. Zero-variance 
    feature regularization achieves the highest accuracy, unattainable via standard tuning alone.}
    \label{fig:combined_max_features2}
  \end{subfigure}
  \hspace*{\fill}
  \caption{Accuracy trends across feature selection sizes ($m$) for artist and digit datasets, highlighting the implicit 
  regularization effects of constant-value features.}
  \label{fig:combined_max_features}
\end{figure}

\topicstart{Scope of model improvements:} We analyze improvements at $m=6$, the setting with the largest discrepancy, for both ticket pricing and digit classification. Figs.~\ref{fig:regularization_improvement_tickets} and~\ref{fig:regularization_improvement_digits} show statistically significant gains from zero-variance regularization. For ticket pricing, $\mathcal{D}_{\alpha \beta}^{\text{(reg)}}$ adds zero-variance features while $\mathcal{D}_{\alpha \beta}$ includes $\alpha$ and $\beta$ only. A paired comparison gives $n_{\text{better}}=1084$, $n_{\text{worse}}=675$, and $N'=1759$. Under $H_0$ with $p=0.5$, a binomial sign test with the normal approximation and continuity correction yields $Z\approx 9.72$ and $p<10^{-21}$. Figure~\ref{fig:regularization_improvement_tickets} confirms this improvement. For digit classification, $\mathcal{D}_{\delta}^{(\text{reg})}$ includes zero-variance features while $\mathcal{D}_{\delta}$ excludes them. The paired counts are $n_{\text{better}}=52$, $n_{\text{worse}}=14$, and $N'=66$. The same test gives $Z\approx 4.56$ and $p<10^{-5}$, as highlighted in Fig.~\ref{fig:regularization_improvement_digits}. Table~\ref{tab:scope_improvements} summarizes the statistical significance for both datasets.

\begin{table}[t]
\centering
\caption{Summary of statistical results for improvements with $m=6$.}
\label{tab:scope_improvements}
\begin{tabular}{lcc}
\toprule
\textbf{Statistic} & \textbf{Tickets} & \textbf{Digits} \\
\midrule
Effective Sample Size ($N'$) & 1759 & 66 \\
$n_{\text{better}}$ & 1084 & 52 \\
$n_{\text{worse}}$ & 675 & 14 \\
Mean ($\mu = N'/2$) & 879.5 & 33 \\
Std. Dev. ($\sigma$) & 20.98 & 4.06 \\
Z-score & 9.72 & 4.56 \\
$p$-value & $<10^{-21}$ & $<10^{-5}$ \\
\bottomrule
\end{tabular}
\end{table}

\begin{figure}[t]
  \centering
  \hspace*{\fill}
  \begin{subfigure}[c]{0.48\linewidth}
    \includegraphics[width=\linewidth]{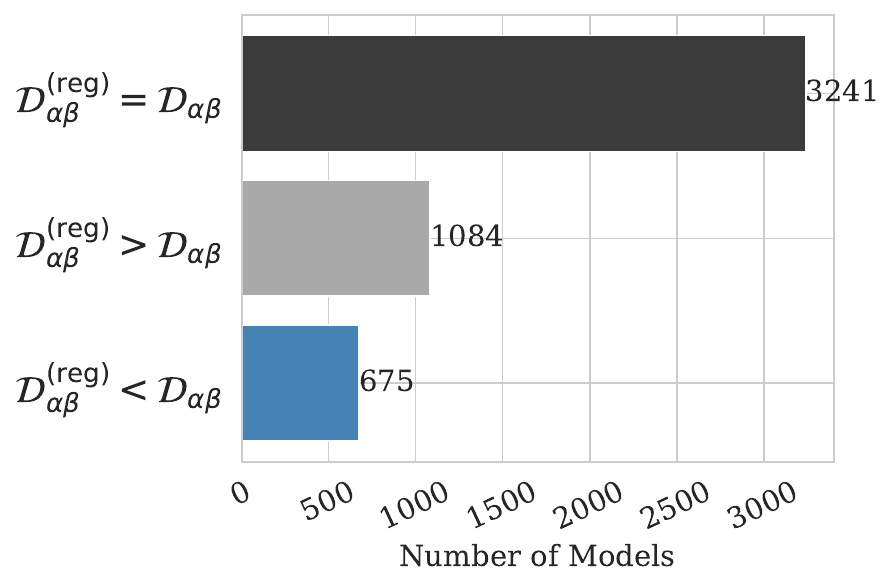}
    \caption{Performance comparison for $m = 6$ across $N_{\text{models}} = 5{,}000$ artist classification
    models. Bars indicate how often $\mathcal{D}_{\alpha \beta}^{\text{(reg)}}$ performed the same, better, or
    worse than $\mathcal{D}_{\alpha \beta}$ (see Fig.~\ref{fig:combined_max_features1}).}
    \label{fig:regularization_improvement_tickets}
  \end{subfigure}
  \hspace*{\fill}
  \begin{subfigure}[c]{0.48\linewidth}
    \includegraphics[width=\linewidth]{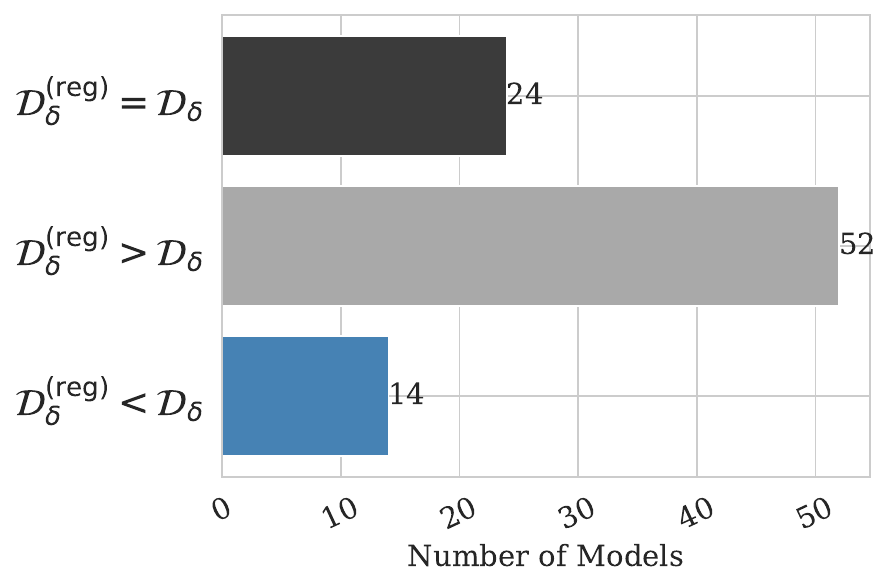}
    \caption{Performance comparison for $m = 6$ across $N_{\text{models}} = 90$ digit classification
    models. Bars indicate how often $\mathcal{D}_{\delta}^{(\text{reg})}$ performed the same, better, or
    worse than $\mathcal{D}_{\delta}$ (see Fig.~\ref{fig:combined_max_features2}).}
    \label{fig:regularization_improvement_digits}
  \end{subfigure}
  \hspace*{\fill}
  \caption{Effect of constant-value feature regularization at $m = 6$, showing the distribution of model
  performance changes for artist and digit classification tasks.}
  \label{fig:regularization_improvement_bars}
\end{figure}

\topicstart{Feature re-ranking and usage in the models:} In both
$\mathcal{D}_{\alpha\beta}$ vs.\ $\mathcal{D}_{\alpha \beta}^{\text{(reg)}}$ and
$\mathcal{D}_{\delta}$ vs.\ $\mathcal{D}_{\delta}^{(\text{reg})}$, introducing
zero-variance features modifies the selection probabilities in the approximate formula
\begin{equation}
P(X_j)
\;\approx\;
\frac{m}{n_{\mathrm{eff}}}
\,\cdot\,
\frac{
  r(X_j)
}{
  r(X_j)
  + (m-1)\,\overline{r},
}
\end{equation}
where $n_{\mathrm{eff}}=n + n_{\mathrm{ZV}}$ when constant features are present and equals
$n$ otherwise. By injecting low-scoring features into the ensemble, these datasets effectively
dilute the dominance of highly ranked predictors $X_j$. Consequently, the final usage
distribution, aggregated across all base learners, becomes more balanced, giving subtle but
informative features more opportunities at split nodes. This re-ranking serves as a form of
implicit regularization, stabilizing the Random Forest. Figs.~\ref{fig:regularization_feature_usage_tickets}
and~\ref{fig:regularization_feature_usage_digits} show the redistribution of feature
usage, highlighting the increased prominence of moderately-ranked predictors. Notably, the shifts
shown in these figures underscore how implicit regularization effectively promotes model robustness
through enhanced feature breadth.

\begin{figure}[t]
  \centering
  \hspace*{\fill}
  \begin{subfigure}[c]{0.48\linewidth}
    \includegraphics[width=\linewidth]{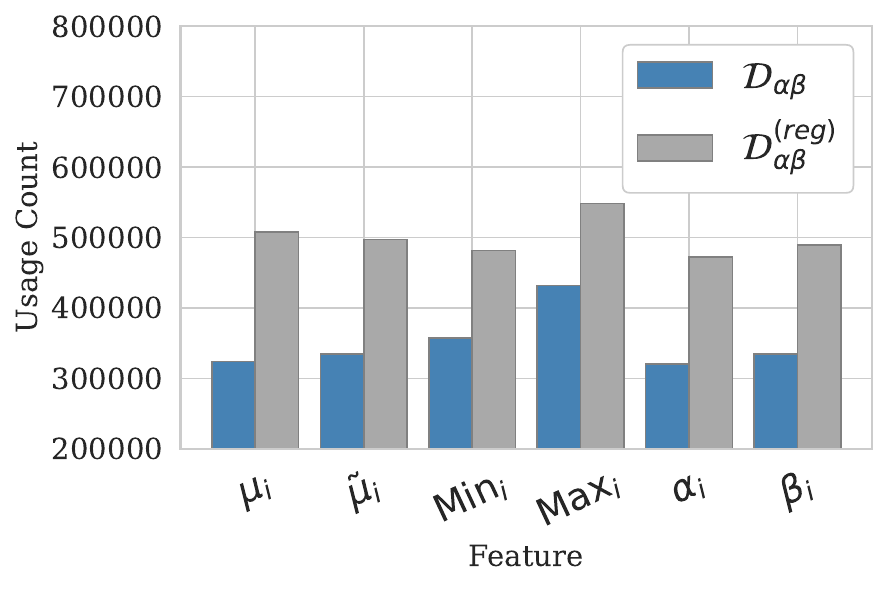}
    \caption{Average feature usage counts across $N_{\text{models}} = 5{,}000$ for
    $\mathcal{D}_{\alpha \beta}$ versus $\mathcal{D}_{\alpha \beta}^{\text{(reg)}}$.}
    \label{fig:regularization_feature_usage_tickets}
  \end{subfigure}
  \hspace*{\fill}
  \begin{subfigure}[c]{0.48\linewidth}
    \includegraphics[width=\linewidth]{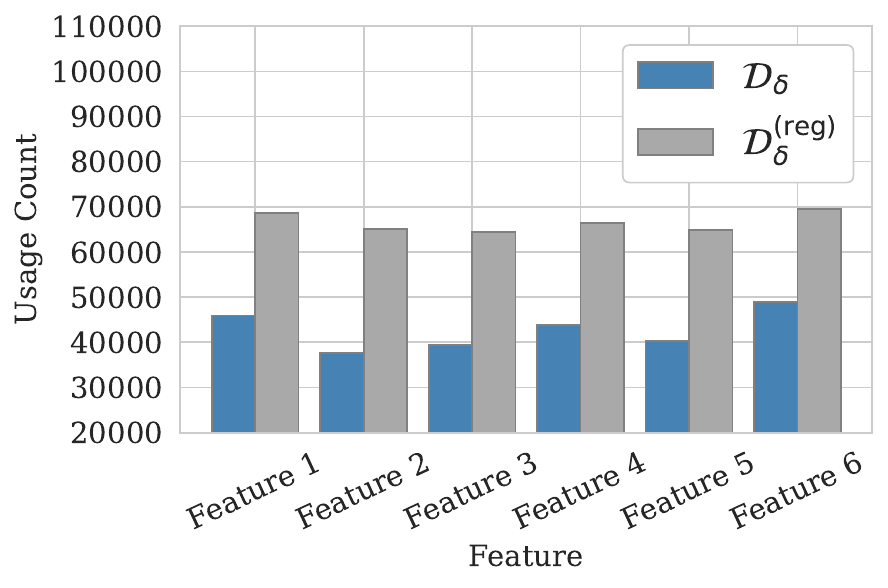}
    \caption{Average feature usage counts across $N_{\text{models}} = 90$ for
    $\mathcal{D}_{\delta}$ versus $\mathcal{D}_{\delta}^{(\text{reg})}$.}
    \label{fig:regularization_feature_usage_digits}
  \end{subfigure}
  \hspace*{\fill}
  \caption{Comparison of average feature usage between unregularized and regularized models, across both artist
  and digit datasets. Regularization via zero-variance features leads to more varied and balanced feature
  selection.}
  \label{fig:regularization_feature_usage}
\end{figure}

\topicstart{Increased tree depth with implicit regularization:} We further examine zero-variance features by computing the average tree depth across all models. Models with zero-variance features ($\mathcal{D}_{\delta}^{(\text{reg})}$ and $\mathcal{D}_{\alpha \beta}^{\text{(reg)}}$) grow deeper trees than non-regularized counterparts, consistent with corollary~\ref{corollary:increased_tree_depth} and supporting the view that implicit regularization encourages the use of more nuanced feature representations. The results validate our expected-depth analysis and extend \citep{geurts2006extremely}, where randomized split selection increases depth through reduced impurity gains. Figs.~\ref{fig:regularization_tree_depth_tickets} and~\ref{fig:regularization_tree_depth_digits} show the depth increases. For ticket pricing, the median depth rose from 3.0 to 4.0 and the average from 3.18 to 4.16. For handwritten digits, the median rose from 8.0 to 10.0 and the average from 8.35 to 10.42. These consistent and significant increases highlight a stabilizing effect of implicit regularization across distinct datasets.

\begin{figure}[t]
  \centering
  \hspace*{\fill}
  \begin{subfigure}[c]{0.48\linewidth}
    \includegraphics[width=\linewidth]{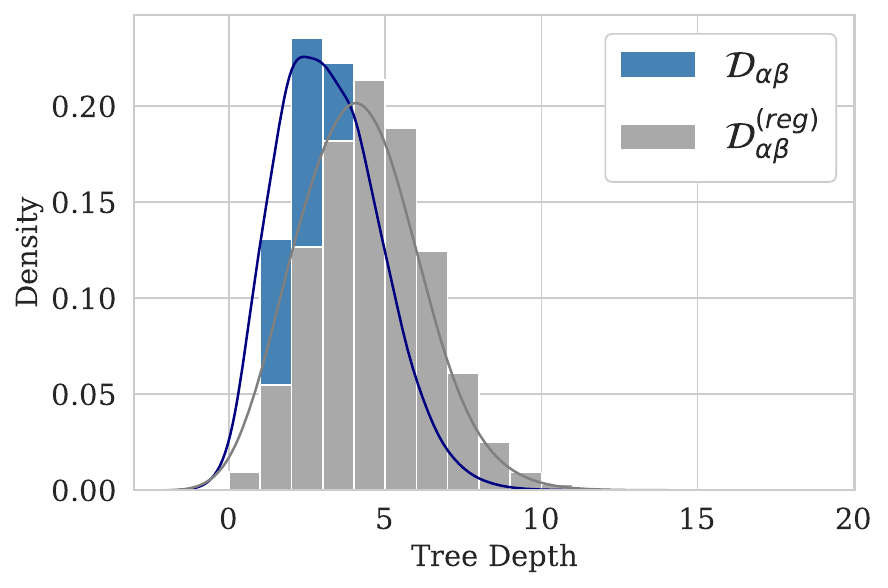}
    \caption{Average tree depth for $N_{\text{models}} = 5{,}000$ models:
    $\mathcal{D}_{\alpha \beta}$ vs.\ $\mathcal{D}_{\alpha \beta}^{\text{(reg)}}$. Median depth increased from
    3.0 to 4.0, and average depth from 3.18 to 4.16.}
    \label{fig:regularization_tree_depth_tickets}
  \end{subfigure}
  \hspace*{\fill}
  \begin{subfigure}[c]{0.48\linewidth}
    \includegraphics[width=\linewidth]{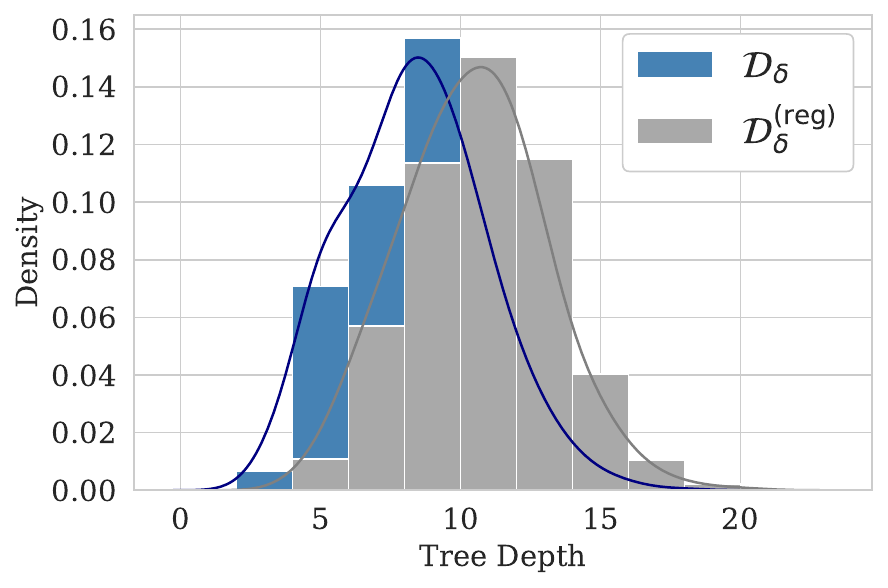}
    \caption{Average tree depth for $N_{\text{models}} = 90$ models:
    $\mathcal{D}_{\delta}$ vs.\ $\mathcal{D}_{\delta}^{(\text{reg})}$. Median depth increased from 8.0 to 10.0,
    and average depth from 8.35 to 10.42.}
    \label{fig:regularization_tree_depth_digits}
  \end{subfigure}
  \hspace*{\fill}
  \caption{Effect of zero-variance feature regularization on tree depth. Both artist and digit models trained with
  regularized datasets grow deeper trees on average, suggesting increased robustness and feature utilization.}
  \label{fig:regularization_tree_depth}
\end{figure}

\topicstart{Tree variety as measured by feature count distance:} We quantify the ensemble
``variety'' by examining the pairwise Euclidean distance between trees’ feature usage vectors,
$\mathbf{v}_i\in \mathbb{R}^d$. Defining the distance between trees $i$ and $j$ as
$\|\mathbf{v}_i-\mathbf{v}_j\|_2$, we compute the sum over all pairs:
\begin{equation}
V(m)=\sum_{1\le i<j\le n}\|\mathbf{v}_i-\mathbf{v}_j\|_2.
\end{equation}
This measure is computed per model for both regularized and non-regularized datasets (Figs.~\ref{fig:regularization_tree_distance_tickets} and~\ref{fig:regularization_tree_distance_digits}). A higher average $V(m)$ indicates more varied feature usage among trees, consistent with the effect of zero-variance features. The figures show that zero-variance features increase variety, reduce correlation among ensemble members, and stabilize models across scenarios. For ticket pricing, the median variety increased from 2.00 to 2.83 and the average from 2.26 to 2.99. For handwritten digits, the median increased from 4.90 to 7.35 and the average from 5.27 to 7.67.

\citep{geurts2006extremely} show that the fully randomized split selection in  
Extra-Trees reduces correlation among trees, which we quantify via the average cosine  
similarity of their normalized feature usage vectors. In particular, for any two trees with  
vectors $\mathbf{v}_i$ and $\mathbf{v}_j$ (with $\|\mathbf{v}_i\|_2 = \|\mathbf{v}_j\|_2 = 1$),  
we have  
\begin{equation}
\mathbf{v}_i^\top \mathbf{v}_j = 1 - \frac{1}{2}\|\mathbf{v}_i-\mathbf{v}_j\|_2^2.
\end{equation}  
Defining the average correlation $p$ as  
\begin{equation}
p = \frac{2}{n(n-1)}\sum_{1\le i<j\le n} \mathbf{v}_i^\top \mathbf{v}_j,
\end{equation} 
we obtain  
\begin{equation}
p = 1 - \frac{1}{n(n-1)}\sum_{1\le i<j\le n}\|\mathbf{v}_i-\mathbf{v}_j\|_2^2.
\end{equation}
Thus, as our variety measure $V(m)$  
increases, the average correlation $p$ decreases, demonstrating that greater tree variety leads  
to reduced inter-tree correlation. This mathematical relationship aligns with our analysis,  
showing that implicit regularization via zero-variance features promotes a more  
decorrelated ensemble.

\begin{figure}[t]
  \centering
  \hspace*{\fill}
  \begin{subfigure}[c]{0.48\linewidth}
    \includegraphics[width=\linewidth]{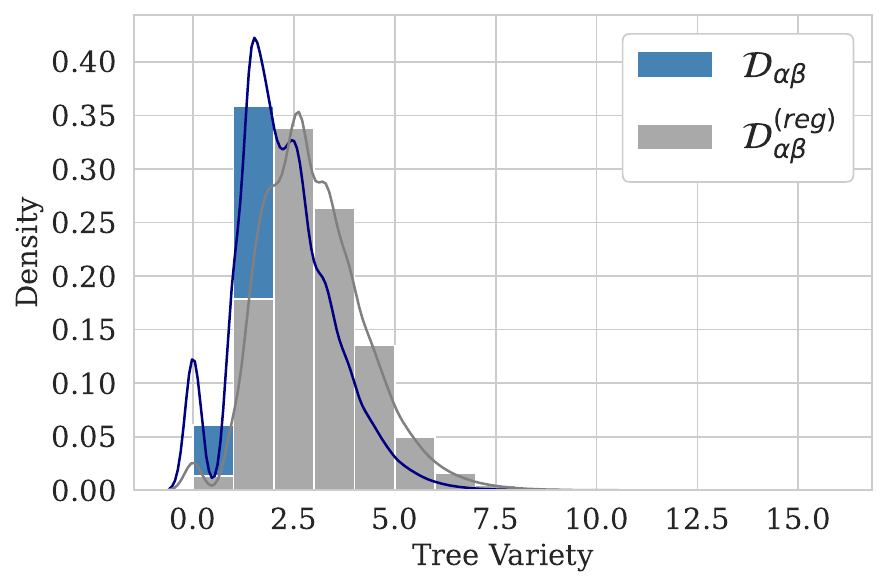}
    \caption{Average tree variety for $N_{\text{models}} = 5{,}000$ models:
    $\mathcal{D}_{\alpha \beta}$ vs.\ $\mathcal{D}_{\alpha \beta}^{\text{(reg)}}$. Median variety increased
    from 2.00 to 2.83, and average variety from 2.26 to 2.99.}
    \label{fig:regularization_tree_distance_tickets}
  \end{subfigure}
  \hspace*{\fill}
  \begin{subfigure}[c]{0.48\linewidth}
    \includegraphics[width=\linewidth]{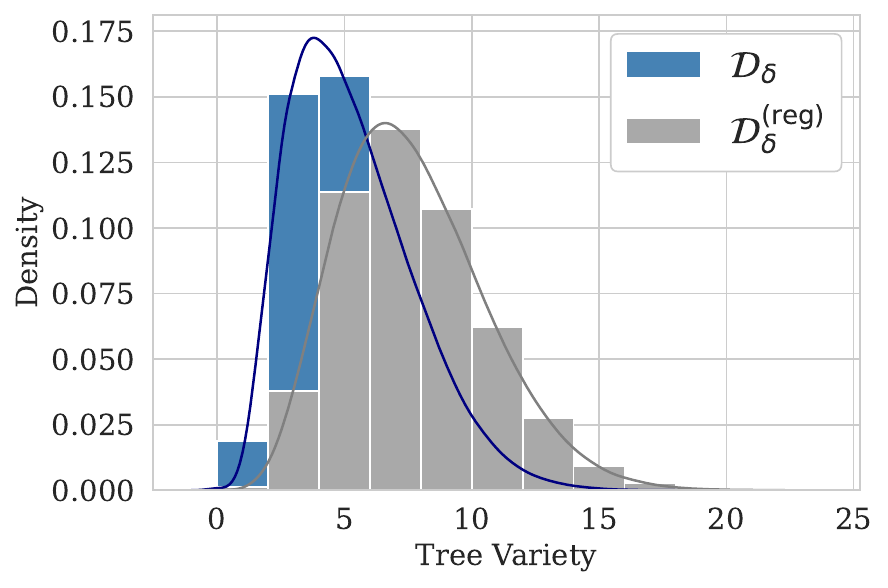}
    \caption{Average tree variety for $N_{\text{models}} = 90$ models:
    $\mathcal{D}_{\delta}$ vs.\ $\mathcal{D}_{\delta}^{(\text{reg})}$. Median variety increased from 4.90 to
    7.35, and average variety from 5.27 to 7.67.}
    \label{fig:regularization_tree_distance_digits}
  \end{subfigure}
  \hspace*{\fill}
  \caption{Effect of zero-variance feature regularization on Random Forest tree variety. Regularization increases
  both the median and average distances of tree structures in the ensemble, improving generalization capacity.}
  \label{fig:regularization_tree_distance}
\end{figure}

\topicstart{Error decomposition and implicit regularization effects:} Recalling the bias-variance decomposition, zero-variance features dilute top-ranked predictors and increase tree variety, which modifies ensemble variance. If $\rho$ is the correlation among tree predictions, then
\begin{equation}
\mathrm{Var}(f(x; n_{\mathrm{ZV}}))
\;=\;
\frac{
  \rho(n_{\mathrm{eff}})\,\mathrm{Var}(h(x; n_{\mathrm{ZV}}))
}{
  B
}.
\end{equation}
This aligns with \citep{geurts2006extremely}, who show that added randomness lowers ensemble correlation. Zero-variance features decrease $\rho(n_{\mathrm{eff}})$ (Corollary~\ref{corollary:reduced_ensemble_correlation}), thus reducing variance. Bias may rise slightly as moderate features are used more, yet our experiments show a net generalization gain.

The regularization view is reinforced by \citep{wyner2017explain}, who attribute AdaBoost's success to averaging interpolating classifiers that yield “spiked-smooth” boundaries, indicating implicit regularization through averaging rather than explicit penalties. This perspective is consistent with the correlation-controlled variance expression above and with classic variance-shrinkage from averaging; see \citep{wyner2017explain}.

\topicstart{Trust in recovered distributions:}  
In the SeatGeek ticket-pricing use case, the experimental results show that implicit regularization not only improves classification accuracy but also expands the effective search space of the Random Forest ensemble. This dual effect enhances trust in the recovered distributions, since accuracy gains more directly validate fidelity to the underlying pricing dynamics. At the same time, varying the number of zero-variance features provides an additional knob for tuning that trust, giving fine-grained control over the balance between distributional fidelity and ensemble robustness.  

\topicstart{Case study, Dropkick Murphys vs. The Avett Brothers:}
A Dropkick Murphys concert on 10/28/2023 is initially mislabeled by the baseline Random Forest as an Avett Brothers show. In the unregularized model the location statistics $\mu_i$, $\tilde{\mu}_i$, $\text{Min}_i$, and $\text{Max}_i$ hold the highest empirical ranks and dominate split lotteries across trees (Fig.~\ref{fig:dm_ab_usage}, blue bars). The artists' feature distributions overlap substantially (Fig.~\ref{fig:kde_alpha_beta_params_Dropkick_Murphys}). Driven by location parameters, the ensemble overlooks the visual match between the event price density (black curve) and the Dropkick template (blue dashed) in Fig.~\ref{fig:dm_ab_density}.

Section~\ref{sec:regularization} adds $n_{\text{ZV}}$ zero-variance columns, diluting the pool from $n$ to $n+n_{\text{ZV}}$ and lowering expected ranks of genuine predictors. Theorem~\ref{thm:zero_variance_dilution} (Zero-Variance Dilution Effect) shows this increases sampling of mid-ranked features, especially $\beta_i$. Empirically $\beta_i$ gains split frequency while location statistics decline (Fig.~\ref{fig:dm_ab_usage}, grey). The amplified shape signal flips a majority of trees and the ensemble classifies the concert correctly.

Previously correct predictions remain. Regularization preserves the decision boundary and informative location cues, gives shape parameters more opportunities, and trims tree correlation for a modest bias-variance gain. The effect requires both scaled beta parameter estimation and zero-variance regularization. Section~\ref{sec:regularization} anticipates this and Figs.~\ref{fig:combined_max_features1} and~\ref{fig:combined_max_features2} show how regularization unlocks a near-continuous space of feature weightings beyond the baseline and standard hyperparameter search.

\begin{figure}[t]
    \centering
    \includegraphics[width=\textwidth]{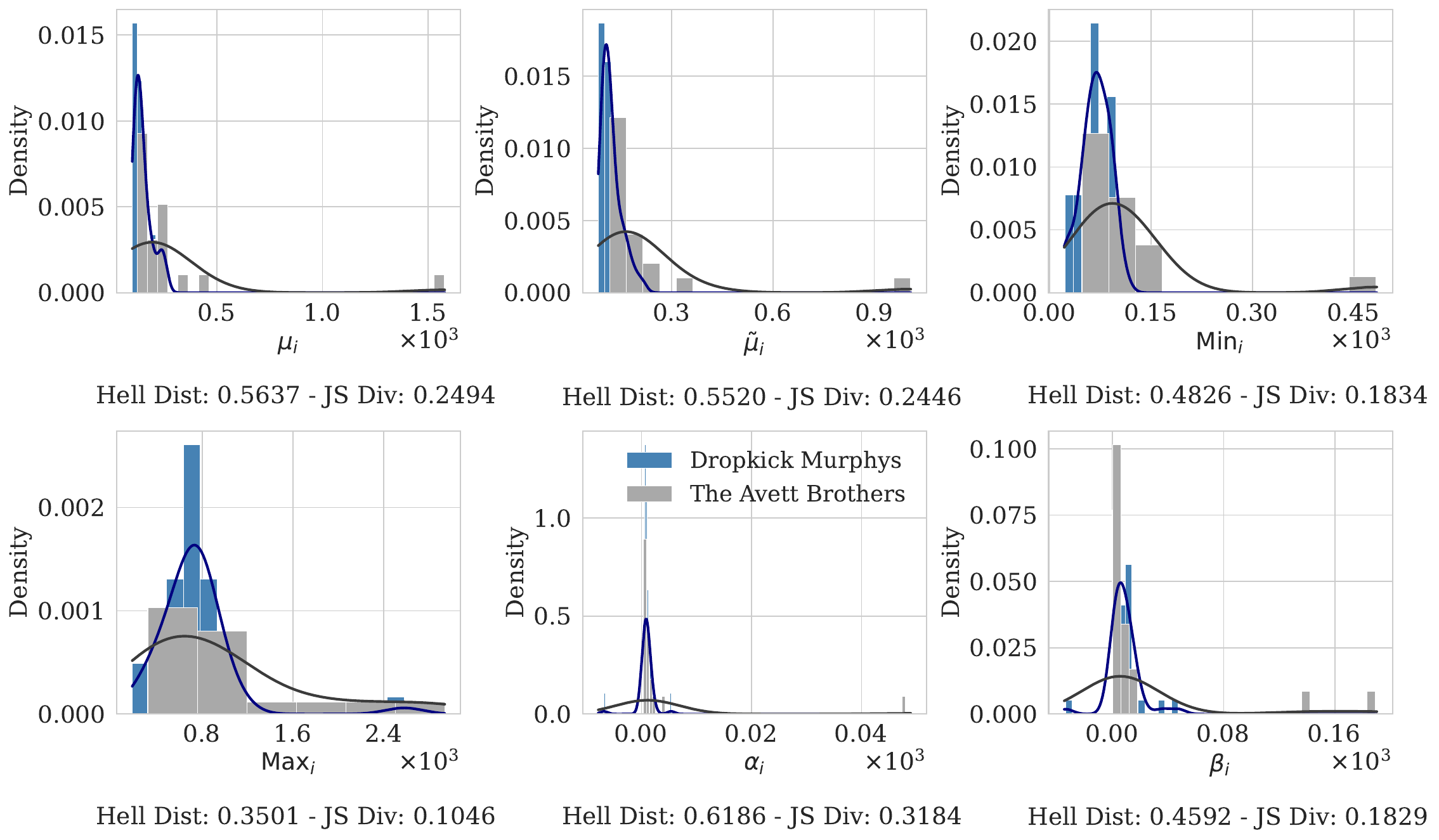}
    \caption{The plots show the distributions of each feature across all events for artists
    Dropkick Murphys and The Avett Brothers. In each case, there is considerable overlap, making classification 
    as determined by location statistics alone difficult. With a satisfactory re-ranking of feature importances by
    the zero-variance implicit regularization mechanism, the Random Forest model can classify more effectively based on 
    the estimated distributional shape as shown explicitly in Fig.~\ref{fig:dm_ab_comparison}.}
    \label{fig:kde_alpha_beta_params_Dropkick_Murphys}
    \vskip -0.2in
\end{figure}

\begin{figure}[t]
  \centering
  \hspace*{\fill}
  \begin{subfigure}[c]{0.48\linewidth}
    \includegraphics[width=\linewidth]{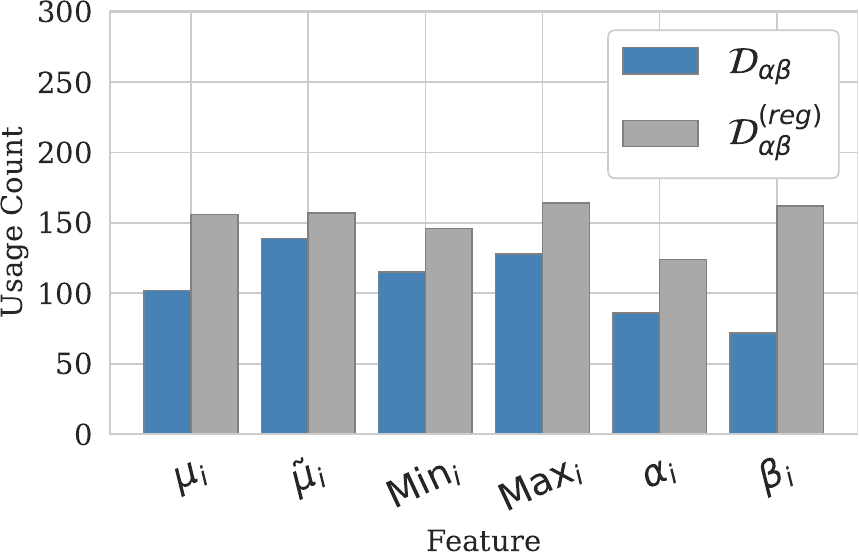}
    \caption{Feature-usage counts across the Random Forest before (blue) and after (grey) the zero-variance
    regularization. Location statistics ($\text{Min}_i$, $\text{Max}_i$, $\mu_i$, $\tilde{\mu}_i$) remain
    frequent but lose relative dominance, while the shape parameter $\beta_i$ increases its appearances in
    split decisions significantly, an explicit illustration of the Zero-Variance Dilution Effect that corrects
    the misclassification and leaves earlier correct calls intact.}
    \label{fig:dm_ab_usage}
  \end{subfigure}
  \hspace*{\fill}
  \begin{subfigure}[c]{0.48\linewidth}
    \includegraphics[width=\linewidth]{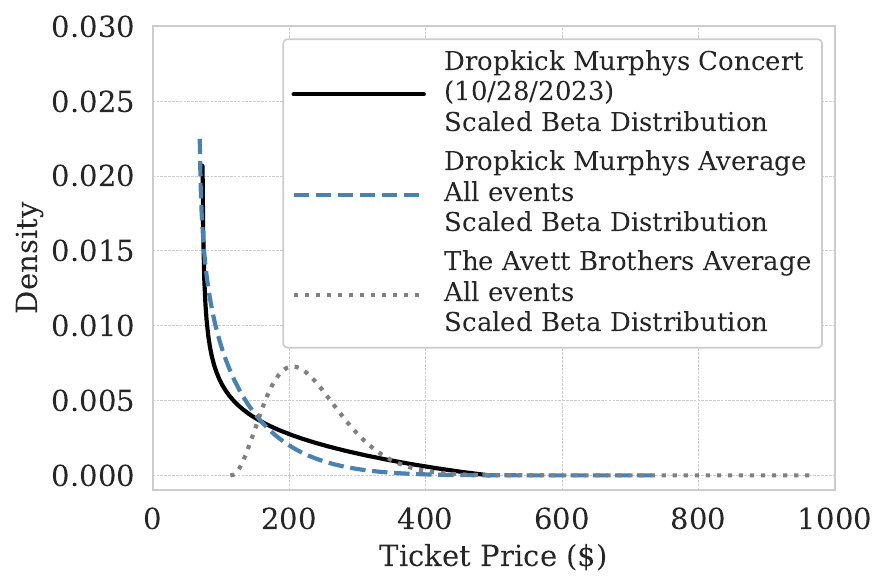}
    \caption{Scaled beta ticket-price densities: the fitted distribution for the 10/28/2023 concert (solid) versus
    the mean Dropkick Murphys profile (dashed) and The Avett Brothers profile (dotted). The signature
    matches the Dropkick template, but the signal is muted when the model relies chiefly on statistics such as
    $\text{Min}_i$, $\text{Max}_i$, $\mu_i$, $\tilde{\mu}_i$.}
    \label{fig:dm_ab_density}
  \end{subfigure}
  \hspace*{\fill}
  \caption{Illustration of how summary statistics can mislead model classification when the true distribution
  shape is more informative. Dropkick Murphys and The Avett Brothers show overlapping location descriptors despite 
  nuanced and differing density shapes. Regularization shifts the focus of the Random Forest to shape.}
  \label{fig:dm_ab_comparison}
\end{figure}

\section{Conclusions}
\label{sec:conclusions}

We conclude by restating and reinforcing the three established contributions.

\begin{enumerate}
    \item \textbf{Closed-form distribution estimation from limited statistics.} We developed a composite quantile-and-moment matching estimator that reconstructs scaled beta distributions from the minimum, maximum, mean, and median $(\text{Min}_i, \text{Max}_i, \mu_i, \tilde{\mu}_i)$, yielding parameters $(\alpha_i, \beta_i)$ that capture shape beyond location summaries. The approach builds on classical beta results \citep{krishnamoorthy2016handbook} and connects to recent work on quantile and moment based estimation under information constraints \citep{zhang2023quantile, dempster2024quant, wei2024latent}. In our application setting, these statistics were retrieved through the SeatGeek API \citep{seatgeek2025}, making the estimator practical at scale. Relatedly, black-box parameter estimation methods~\citep{lenzi2025blackbox} and reliability inference from record values~\citep{pareto2025reliability} highlight the broader importance of recovering distributions from compressed or incomplete data. Unlike iterative solvers, the proposed estimator achieves distributional recovery in a single analytical step, emphasizing both scalability and computational economy.
    \item \textbf{Accuracy–fidelity theory.} We established a link from predictive accuracy to distributional fidelity, using Total Variation Distance and Jensen-Shannon divergence \citep{lin1991divergence, devroye1996probabilistic, tsybakov2004optimal}. The analysis shows that improvements in artist classification accuracy correspond to increasingly precise estimates of the underlying scaled beta parameters, and that convergence of the information-theoretic discrepancy is quadratic in the accuracy margin. This provides stability guarantees for estimation in sparse and noisy environments, a regime that is typical for market snapshots.
    \item \textbf{Implicit regularization via zero-variance features.} We showed that augmenting Random Forests with zero-variance (constant-value) features can serve as an implicit regularizer that reduces the dominance of highly ranked variables, encourages variety, and deepens trees. The effect is consistent with the literature on bagging, random subspaces, and randomized trees \citep{breiman1996bagging, breiman2001random, ho1998rsm, geurts2006extremely}, and resonates with stability and ensemble correlation perspectives \citep{bousquet2002stability, wyner2017explain, wager2018causal}. It complements hyperparameter oriented strategies \citep{mentch2019quantifying, dunbar2025hyperparameter} by shifting split-selection probabilities through small structural changes to the feature space. Comparable to prior work on sparsity-inducing priors~\citep{bai2023sbl}, our contribution emphasizes how structural constraints can implicitly regularize complex ensembles, tuning trust in distributional adherence.
\end{enumerate}

\topicstart{Applied impact and the SeatGeek dataset:} The secondary ticket resale market provides a natural testbed for distributional reconstruction and feature-based classification. Platforms such as SeatGeek, StubHub, and Ticketmaster surface highly dynamic signals at event and artist resolution. Working with the curated SeatGeek dataset and daily snapshots, we transformed subsequences of pricing activity into distributional snapshots and then into learned features that expose artist-specific economic signatures. In the time-series classification literature, targeted feature representations have proven effective for high dimensional problems \citep{middlehurst2024bakeoff, lubba2019catch22, christ2018tsfresh}. Our results contribute a distribution based feature pathway that is tractable from minimal statistics and operational in data settings where full empirical distributions are not retained \citep{seatgeek2025}.

\topicstart{Theoretical guarantees and empirical validation:} The accuracy–fidelity bridge formalized here grounds distribution recovery in risk bounds and information measures \citep{lin1991divergence, devroye1996probabilistic, tsybakov2004optimal}. Within this framework, the modest but consistent classification gains from injecting $(\alpha_i, \beta_i)$ into the feature set have outsized meaning for distributional integrity due to the quadratic Jensen-Shannon rate. In practice, this allows classifier accuracy to serve as an operational stand-in for distributional ground truth when only summary statistics are available. The same perspective clarifies the role of constant-value features as a mechanism that rebalances split selection, lowers correlation across trees, and improves generalization for ensembles \citep{breiman1996bagging, breiman2001random, ho1998rsm, geurts2006extremely, bousquet2002stability, wyner2017explain, wager2018causal, mentch2019quantifying, dunbar2025hyperparameter}. Empirically, we validated the approach on a newly curated SeatGeek pricing dataset and on the UCI handwritten digits benchmark \citep{uci_digits}, confirming generality beyond ticket pricing and showing that the implicit regularization effect is not domain restricted.

\topicstart{Outlook:} The arc demonstrated here is as follows: sparse distributional snapshots of time series $\rightarrow$ closed-form scaled beta estimation $(\alpha,\beta)$ $\rightarrow$ Random Forest accuracy gains, with fidelity, and therefore trust, amplified by implicit regularization from zero-variance features. It is broadly applicable when dynamic systems are observed through compressed summaries. For live-market analytics, healthcare operations, demand forecasting, and energy systems, the same constraints on data access recur. The methods presented are simple to instrument, amenable to scale, and compatible with existing ensemble workflows. In settings where data arrives as aggregated snapshots rather than full samples, this narrative offers a principled route to reconstruct informative distributions, improve classification, and support decision-making with clear theoretical guarantees.

\begin{acknowledgment}

  \topicstart{Code and data availability:}
  Datasets, methodology code, and figure generation code are available 
  at the following GitHub repository for reproducibility:\\
  \href{https://github.com/jonland82/seatgeek-beta-modeling}{\texttt{github.com/jonland82/seatgeek-beta-modeling}}

  \topicstart{Methods:}
  The author used standard computational tools and programming libraries,
  including Python packages and a large language model (OpenAI),
  to assist with code snippets, algebraic manipulation, and editorial refinement.
  All content and interpretations were conceptualized, reviewed and finalized by the author.

  \topicstart{Declarations:}
  The author used publicly available event data accessed via the SeatGeek API (SeatGeek, Inc.) in
  accordance with SeatGeek’s API Terms of Use. SeatGeek is not affiliated with this research and does not 
  endorse it. All trademarks and content remain the property of their respective owners. Proper attribution 
  is provided at \href{https://seatgeek.com}{seatgeek.com} as required. Raw API data is not redistributed 
  per licensing requirements.
  \\
  \\
  The author reports no conflicts of interest. No funding was received for this research.

\end{acknowledgment}

\bibliographystyle{plainnat}   
\bibliography{references}

\end{document}